\documentclass[lettersize,journal]{IEEEtran}

\usepackage{amsmath,amsfonts}
\usepackage{algorithmic}
\usepackage{algorithm}
\usepackage{amsthm}
\usepackage{array}
\usepackage[caption=false,font=normalsize,labelfont=sf,textfont=sf]{subfig}
\usepackage{textcomp}
\usepackage{stfloats}
\usepackage{url}
\usepackage{verbatim}
\usepackage{graphicx}
\usepackage{cite}
\usepackage{tikz,pgfplots}

\usepackage{import}     
\usepackage{caption}

\def\caln{{\cal N}}
\def\calg{{\cal G}}
\def\cale{{\cal E}}
\def\cali{{\cal I}}
\def\calo{{\cal O}}
\def\calh{{\cal H}}

\def\calm{{\cal M}}
\def\calr{{\cal R}}
\def\calj{{\cal J}}

\def\cale{{\cal E}}

\def\calp{{\cal P}}

\def\calj{{\cal J}}

\def\calx{{\cal X}}

\def\ug{U_{\tt G}}
\def\ue{U_{\tt E}}
\def\up{u_{\tt P}}

\def\L2e{{\cal L}_{2e}}

\def\rea{\mathbb{R}}

\def\sign{\mbox{sign}}

\def\diag{\mbox{diag}}

\def\col{\mbox{col}}
\def\hal{{1 \over 2}}

\def\diag{\mbox{diag}}

\def\min{{\mbox{min}}}

\def\begmat#1{\begin{bmatrix}#1\end{bmatrix}}

\def\begquo{\begin{quote}}
\def\endquo{\end{quote}}
\def\begequarr{\begin{eqnarray}}
\def\endequarr{\end{eqnarray}}
\def\begequarrs{\begin{eqnarray*}}
\def\endequarrs{\end{eqnarray*}}
\def\begarr{\begin{array}}
\def\endarr{\end{array}}
\def\begequ{\begin{equation}}
\def\endequ{\end{equation}}

\def\begdes{\begin{description}}
\def\enddes{\end{description}}
\def\begenu{\begin{enumerate}}
\def\begite{\begin{itemize}}
\def\endite{\end{itemize}}
\def\endenu{\end{enumerate}}
\def\lef[{\left[\begin{array}}
\def\rig]{\end{array}\right]}
\def\begcen{\begin{center}}
\def\endcen{\end{center}}
\def\endrem{\end{remark}}
\def\begdef{\begin{definition}}
\def\enddef{\end{definition}}
\def\begpro{\begin{proposition}}
\def\endpro{\end{proposition}}
\def\begfac{\begin{fact}}
\def\endfac{\end{fact}}
\def\begsubequ{\begin{subequations}}
\def\endsubequ{\end{subequations}}

\newcommand{\red}[1]{{\color{red} #1}}

\usepackage{booktabs} 

\newtheorem{definition}{Definition}
\newtheorem{assumption}{Assumption}

\newtheorem{remark}{Remark}

\newtheorem{proposition}{Proposition}

\definecolor{amethyst}{rgb}{0.6, 0.4, 0.8}
\definecolor{carminepink}{rgb}{0.92, 0.3, 0.26}
\definecolor{color1}{rgb}{0.196, 0.722, 0.592}
\definecolor{color2}{HTML}{D76364}
\definecolor{babyblue}{rgb}{0.54, 0.81, 0.94}
\usepackage{multirow}
\usepgfplotslibrary{fillbetween}
\usepackage[colorlinks=true, linkcolor=blue, citecolor=blue, urlcolor=blue]{hyperref}

\begin{document}

\title{Modeling, Control, and Stiffness Regulation of Layer Jamming-based Continuum Robots}

\author{Yeman Fan, Bowen Yi, and Dikai Liu
\thanks{This work was supported in part by Australian Research Council (ARC) Discovery Project under Grant DP200102497, the Robotics Institute of the University of Technology Sydney, Natural Sciences and Engineering Research Council of Canada (NSERC), and the Programme PIED of Polytechnique Montr\'eal. (Bowen Yi and Yeman Fan contributed equally to this work.) (Corresponding author: Bowen Yi.)

Bowen Yi is with the Department of Electrical Engineering, Polytechnique Montr\'eal and GERAD, Montr\'eal, QC H3T 1J4, Canada (bowen.yi@polymtl.ca)

Yeman Fan and Dikai Liu are with the Robotics Institute, Faculty of Engineering and Information Technology, University of Technology Sydney, Sydney, NSW 2007, Australia (e-mail: yeman.fan@uts.edu.au; dikai.liu@uts.edu.au).}
\thanks{Manuscript received 2025.}}



\maketitle

\begin{abstract}
Continuum robots with variable compliance have gained significant attention due to their adaptability in unstructured environments. Among various stiffness modulation techniques, layer jamming (LJ) provides a simple yet effective approach for achieving tunable stiffness. However, most existing LJ-based continuum robot models rely on static or quasi-static approximations, lacking a rigorous control-oriented dynamical formulation. Consequently, they are unsuitable for real-time control tasks requiring simultaneous regulation of configuration and stiffness and fail to capture the full dynamic behavior of LJ-based continuum robots. To address this gap, this paper proposes a port-Hamiltonian formulation for LJ-based continuum robots, formally characterizing the two key phenomena---shape locking and tunable stiffness---within a unified energy-based framework. Based on this model, we develop a passivity-based control approach that enables decoupled regulation of stiffness and configuration with provable stability guarantees. We validate the proposed framework through comprehensive experiments on the OctRobot-I continuum robotic platform. The results demonstrate consistency between theoretical predictions and empirical data, highlighting the feasibility of our approach for real-world implementation. 
\end{abstract}


\section{Introduction} \label{sec:1}
\subsection{Background and Related Work}
Continuum robots are important in many applications due to their inherent flexibility and dexterity, such as minimally invasive surgeries, whole-arm grasping, and search-and-rescue operations \cite{Burgner}. When interacting with the environment or humans, there is a need to actively change the dynamical response of the robots, particularly mechanical impedance or hybrid motion/force \cite{BAJSIM}. Indeed, many continuum robots have integrated variable stiffness techniques within their design, allowing them to switch between soft, flexible motion and rigid resistance, a process known as stiffening. This adaptability significantly expands their range of applications, enhancing their versatility across diverse domains \cite{NARetalRAL, CLAROJ, YANetal}.

Among the various stiffening techniques, jamming-based approaches have shown great success in achieving adjustable stiffness in continuum robots, providing rapid and reversible responses \cite{CLAROJtro, SANetal, LANetal}. Based on the jamming materials used, these methods can be broadly classified into fiber, granular, fabrics, and layer jamming. A comprehensive comparison of these techniques can be found in \cite{FAN24RCIM}. Notably, the concept of layer jamming (LJ) in continuum robots was explored earlier in \cite{KIMetalIROS,KIMetal} and has since received particular attention due to its advantages of being lightweight and space-efficient. In LJ-based continuum robots, an airtight pneumatic chamber contains a series of overlapping jamming flaps, typically made of plastic or paper layers, which either encase the robot's spine or form its body;~see Fig.~\ref{fig:jamming}. This mechanism exploits the friction between layers that can be controlled by external pressure via a \emph{vacuum}. Importantly, in addition to mechanical actuation (for configuration control), this pressure can be seen as an additional control input for LJ-based continuum robots, providing greater flexibility for real-time tasks. As a result, layer jamming provides a broad range of controllable stiffness.

\begin{figure}[!h]
    \centering
    \includegraphics[width = 0.93\linewidth]{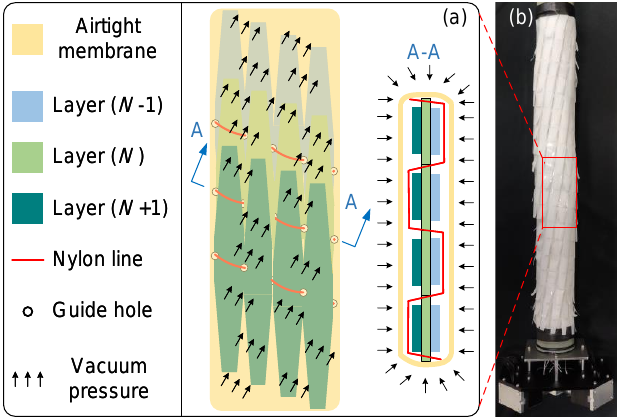}
    \caption{Schematic of a layer-jamming structure.}
    \label{fig:jamming}
\end{figure}

Real-time dynamic control is one of the most important topics in the field of robotics, where it is well-established for rigid-link robots \cite{SCISIC} but still in its infancy for continuum robots. Previously, modeling the dynamic behavior of a continuum robot was considered overwhelmingly complicated. However, over the past few years, \emph{model-based control} approaches have gained resurgence, as more experimental evidence has shown that feedback control is robust to the approximations used in modeling of continuum robotic dynamics \cite{DELetalREV}. Some very recent works have further explored the dynamic modeling and model-based dynamic control of continuum robots, employing models such as the rigid-link approximation models \cite{DELetal,FRAGAR} and the dynamic Cosserat model \cite{CAAetal}. Notably, \cite{DELetal} presents the first solution in the literature on dynamic feedback control in trajectory tracking and surface following of a planar soft robot.

In addition to the aforementioned configuration control tasks, another important topic for continuum robots is stiffness modulation \cite{KIMetal,LANetal}. In many cases, it is necessary to regulate both the configuration and stiffness of the LJ-continuum robot simultaneously in real time, with the jamming mechanism adding complexity to dynamic modeling and control. As openly recognized by one of the leading groups in the field, the modeling of LJ-based continuum robots has not been well addressed yet \cite{NARetal}:

\begin{quote}
   \it ``Nevertheless, these studies have not yet provided analytical or computational models for laminar jamming beyond an initial deformation phase ...''
\end{quote}
%
The absence of such models can hardly be overestimated: on one hand, it limits our understanding of the dynamic behavior of LJ-based continuum robots; on the other hand, it restricts our ability to design high-performance dynamic controllers for adjusting the robot's configuration and stiffness. 

To address the challenge, recent studies have developed analytical and computational models to characterize stiffness variation in LJ-based continuum robots, such as  \cite{ZHAetalBB,KIMetal,CHENetal,CARetal}. Most of these existing works focus on \emph{static} models, and notably the recent work \cite{DOetal} investigates \emph{quasistatic} modeling for a soft jamming brake and artificial muscle's operating modes and stiffness. While static and quasistatic models provide insights into the stiffness modulation of LJ-based continuum robots, they lack the necessary dynamical representation required for real-time feedback control. In \cite{IBRetal}, the authors design a stiffness- and shape-changing device based on a mechanism with a multiple-chamber structure; however, the transition phases are uncontrollable, and they point out that ``a control strategy of the phase transitions needs to be explored in depth''. Clearly, a reliable \emph{control-oriented dynamical} model for LJ-based continuum robots is needed. To the best of the authors' knowledge, no such model currently exists in the literature to characterize the dynamic behaviors of LJ-based continuum robots. Consequently, the problem of feedback control with guaranteed stability, whether model-based or data-driven, of LJ-based continuum robots is still an open problem. Addressing this challenge will provide a systematic algorithmic solution to simultaneously regulate the stiffness and configuration of LJ-based continuum robots in a precise and quantitative manner. 


In our previous work \cite{YIetal23}, we explored the feedback mechanism for dynamically adjusting the closed-loop stiffness of a class of underactuated tendon-driven continuum robots. However, we do not utilize the
jamming mechanism, and as a result, the attainable range
of this stiffness is constrained within a relatively narrow
interval due to its inherent properties. In this paper, we \emph{incorporate layer jamming} into the dynamic modeling and model-based control of continuum robots to achieve significantly broader tunability of closed-loop stiffness with enhanced performance. This is a non-trivial challenge: on one hand, even \emph{static} modeling of the jamming mechanism is inherently complex \cite{NARetal,CARetal}; on the other hand, in dynamic modeling, we must account for the intricate coupling between the stiffness mechanism and the dynamics. Notably, the phenomenon of \emph{shape locking} is frequently observed in LJ-based continuum robots \cite{ZHAetalBB,NARetal,SANetal}. This phenomenon refers to the robot's shape becoming constrained or ``locked'' due to increased interlayer friction when the vacuum pressure reaches a critical level. A reliable dynamical model must not only capture this effect but also provide an interpretable framework.

We aim to solve the following two problems for LJ-based continuum robots:
\begin{itemize}\setlength{\leftskip}{-3pt}
    \item[P1] \emph{Control-oriented dynamic modeling}: Develop a dynamical model that:
    \begin{itemize}\setlength{\leftskip}{-20pt}
        \item[(i)] is simple and interpretable, facilitating the design of a dynamic controller for simultaneous regulation of stiffness and configuration;
        \item[(ii)] captures the phenomena of shape locking and variable stiffness, ensuring compatibility with existing static modeling results.  
    \end{itemize}
\item[P2] \emph{Decoupled control of configuration and stiffness}: Based on the dynamic model design a controller that:  
\begin{itemize}  \setlength{\leftskip}{-20pt}
    \item[(iii)] has a compact form while ensuring robust and accurate control performance;  
    \item[(iv)] treats vacuum pressure as an additional control input to achieve decoupled regulation of configuration and closed-loop stiffness.  
\end{itemize}  
\end{itemize}

We adopt the energy-based approach \cite{VAN,ORTetalBOOK,FERetal,YIetalAUT}, which has proven highly effective for both modeling and real-time control of continuum robots \cite{DELetal,FRAGAR,BORetal,PAGetal}. This approach was recently applied to unified force-impedance control of serial rigid and flexible-joint robots \cite{HADSHA} by using the concept of energy tank. We propose a port-Hamiltonian formulation for LJ-based continuum robots, formally characterizing the two key phenomena---shape locking and tunable stiffness---and then enabling real-time control within a unified energy-based framework.

\subsection{Contributions}
The main contributions of this work are summarized below:

\begin{itemize} \setlength{\leftskip}{-3pt} 
    \item[1.] We propose a novel control-oriented dynamic model for a class of LJ-based, tendon-driven continuum robots by integrating the energy-based modeling technique and the LuGre frictional model \cite{ASTDEW}. The resulting overall model is in the port-Hamiltonian form, with the vacuum pressure as an additional control input.

    \item[2.] We comprehensively analyze the proposed model and theoretically prove its capability to interpret two important phenomena of the LJ-based continuum robots: shape locking and adjustable stiffness. Utilizing the model, we provide an analytic formula for the relationship between stiffness and negative pressure, precisely aligning with our experimental results.

    \item[3.] Adaptively adjusting stiffness and achieving precise configuration control are two fundamental tasks for LJ-based continuum robots. To the best of the authors' knowledge, we propose the first feedback control approach in the literature to achieve both perspectives in a decoupling manner. Notably, the closed-loop stiffness is dominated by the negative pressure---serving as an additional control input---in the jamming sheath.
\end{itemize}

The theoretical aspects discussed above have been validated through comprehensive experiments. In particular, our findings highlight the effectiveness and superiority of the proposed modeling and control approaches. 

A preliminary version of this work, including partial results from Sections \ref{sec:2} and \ref{sec:3}, was presented at ICRA 2024 \cite{YIetalICRA}. 
This extended version significantly advances our previous work by refining the nonlinear stiffness modeling, designing a novel feedback controller with provable decoupling properties, and validating the approach through extensive closed-loop experiments. Specifically, we 1) refine the nonlinear effects arising from negative pressure in the LuGre friction model and the elastic potential energy, leading to improved model accuracy (Section~\ref{sec:2}); 2) develop a feedback controller that achieves decoupled regulation of both configuration and stiffness (Section~\ref{sec:4}); 3) conduct closed-loop experiments to validate the proposed controller and present additional experimental results (Section~\ref{sec:5}).

\subsection{Paper Organization and Notation}
\textit{Organization.} The subsequent sections of the paper are structured as follows. In Section \ref{sec:2}, we provide some preliminaries of essential concepts related to the modeling of continuum robots, as introduced in \cite{YIetal23}. Additionally, we elucidate the modifications necessary to incorporate layer jamming into this model. Subsequently, in Section \ref{sec:3}, we demonstrate the effectiveness of the proposed model in theoretically interpreting two key phenomena observed in LJ-based continuum robots. Section \ref{sec:4} presents the feedback control design, followed by the presentation of experimental results in Section \ref{sec:5}. The paper is wrapped up by some concluding remarks in Section \ref{sec:6}.  

\textit{Notation.} All functions and mappings are assumed to be $C^2$-continuous. $I_n$ is the $n \times n$ identity matrix, $0_{n \times s}$ is an
$n \times s$ matrix of zeros, and $\mathbf{1}_n := \col(1, \ldots, 1) \in \rea^n$. We also use $\mathbf{1}_{n\times n}$ to represent an $n\times n$ matrix of ones. For $x \in \rea^n$, $S \in \rea^{n \times n}$, $S=S^\top
>0$, we denote the Euclidean norm $|x|^2:=x^\top x$, and the weighted--norm $\|x\|^2_S:=x^\top S x$. Given a function $f:  \rea^n \to \rea$, we define the differential operators
$
\nabla f:=(\frac{\partial f }{ \partial x})^\top,\;\nabla_{x_i} f:=(\frac{
\partial f }{ \partial x_i})^\top,
$
where $x_i \in \rea^p$ is an element of the vector $x$. The set $\bar n$ is defined as $\bar n:= \{1,\ldots,n\}$. We use $\diag\{x_i\}~(i\in\bar n)$ to represent the diagonal matrix $\diag\{x_1, \ldots, x_n\}$, and define the set $B_\varepsilon(\calx) := \{x \in \rea^n: \inf_{y \in \calx} |x-y| \le \varepsilon \}$ for a given set $\calx \in \rea^n$. When clear from the context, the arguments of the functions and signals may be omitted.

\begin{figure}[!htp]
\fbox{\parbox {.97\linewidth}
{
\begin{center}
{\sc Nomenclature}
\end{center}
\renewcommand\arraystretch{1.2}
\setlength{\tabcolsep}{5pt}
\small
\begin{tabular}{l l}
$k_n \in \mathbb{N}_+$ & Number of layers\\
$q\in \rea^n$ & Configuration variable\\
$p\in \rea^n$ & Generalized momentum\\
$z \in \rea^n$ & Virtual bristle deflection at each (virtual) joint\\
$v\in \rea^n$ & Relative generalized velocity of surfaces in contact
\\
$\chi \in \rea^{3n}$ & Overall system state $\chi=\col(q,p,z)$\\
$D(q) \succeq0 $ & Damping matrix \\
$G(q)$ & Input matrix\\
$\up \in \rea_{\ge 0}$ & Negative pressure \emph{value} in the sheath \\
$u \in \rea^m$ & Tension input\\
$U_{\tt E}(q)$ & Elastic potential energy\\
$U_{\tt G}(q)$ & Gravitational potential energy\\
$U(q)$ & Total potential energy \\
$H(q,p)$ & Hamiltonian function of the continuum robot\\
$\calh(\chi, \up)$& Overall Hamiltonian with frictional models \\
$M(q)$ & Inertia of the robot 
\end{tabular}
}}
\end{figure}

\section{Dynamic Modeling} \label{sec:2}
\subsection{Preliminary of Jamming-free Model} \label{sec:21}
In this section, we provide a brief overview of the dynamical model of continuum robots \emph{without} layer jamming. In the subsequent sections, we will explore the dynamic modeling of LJ-based continuum robots.

We consider the control-oriented modeling of a class of underactuated tendon-driven continuum robots and concentrate on the two-dimensional case. A rigid link model is used to approximate the dynamical behavior of continuum robots as follows \cite{FRAGAR,YIetal23}:
\begin{equation}
\label{eq:model}
\begmat{\dot q \\ ~\dot p~}
= 
\begmat{~0_{n\times n} & I_n \\ - I_n  & - D(q)~}
\begmat{\nabla_q H \\ ~\nabla_p H~} 
+
\begmat{0_n \\~ G(q) u~}
\end{equation}
with the configuration variable $q \in \calx \subset \rea^n$ including all the segment angles, the generalized momentum $p \in \rea^n$, the input matrix $G:\rea^n \to \rea^{n\times m}$, the damping matrix $D(q) \in \rea^{n\times n}_{\succeq 0}$, and the tension input $u~\in~\rea^m$ ($n,m \in \mathbb{N}_+$). We consider the case of underactuation, i.e., $n > m$. The total energy is characterized by the Hamiltonian 
\begin{equation}
\label{eq:H}
H(q,p) = \frac{1}{2} p^\top M^{-1}(q) p + U(q),
\end{equation}
consisting of kinetic and potential energy,
where $M: \rea^n \to \rea^{n\times n}_{\succ 0}$ is the positive definite inertial matrix satisfying $m_1 I \preceq M(q) \preceq m_2 I $ for some $m_2\ge m_1>0$. Here, the potential energy function $U(q)$ contains the gravitational part $U_{\tt G}$ and the elastic part $U_{\tt E}$ that are functions of $q$, and satisfies 
\begin{equation}
\label{eq:U}
U(q) = U_{\tt G}(q) + U_{\tt E}(q).    
\end{equation}

As shown in Fig.~\ref{fig:pcc}, under the geometric assumptions of piecewise constant curvature (PCC) for the continuum robot, where the masses and lengths are uniform along the robot body, and that the axial deformation resulting from antagonistic tensions is negligible compared to the bending, the potential energy functions can be modeled as
\begin{equation}
\label{eq:UgUe}
\begin{aligned}
    \ug  & 
    ~=~ \alpha_1 [1- \cos(q_\Sigma)] \\
    \ue  & 
    ~=~ \hal q^\top \Lambda q + U_0
\end{aligned}
\end{equation}
with
the diagonal matrix 
$
\Lambda := \diag\{\alpha_2, \ldots, \alpha_2\}
$
and some positive scalar $U_0$, and for simplicity we define a new variable
$ q_\Sigma:= \sum_{i\in \bar n} q_i. $

Note that $\alpha_1$ and $\alpha_2$ are some elastic and gravitational coefficients, respectively. We refer the interested reader to \cite{YIetal23} for additional details about the robotic structure and its energy-based modeling procedure. 

\begin{figure}[!h]
    \centering
    \includegraphics[width = 1\linewidth]{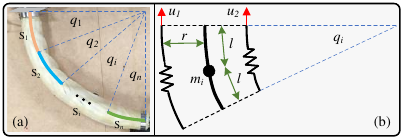}
    \caption{Illustration of the PCC assumption (a) and configuration variables (b). $S_i$ represents a constant curvature segment, $l$ is half length of a segment, $r$ is the radius of the robot, and $m_i$ is the lumped mass.}
    \label{fig:pcc}
\end{figure}
\subsection{Assumptions and Frictional Model with Layer Jamming}
In continuum robots, the layer jamming technique provides a lightweight and rapid response approach to adjust the robots' stiffness \cite{NARetal,CHOetal,JADetal,KIMPARetal}. In this type of robots, layer jamming -- consisting of a laminate of flexible strips or sheets -- is installed throughout the continuum robot's body, and winded up into a tube sheath, as illustrated in Fig. \ref{fig:jamming}. The jamming sheath forms an enclosed structure in which we may apply a negative pressure $-u_{\tt P} \le 0$ (relative to the atmospheric pressure) \cite{FANetal}. As a result, friction between strips or sheets would increase dramatically, thus changing the robot stiffness and dissipating energy \cite{NARetal}.

In our previous work \cite{YIetal23}, we operate the robot in its default state (i.e. $u_{\tt P}=0$) that behaves highly compliant, without accounting for the effects of negative pressure.
As discussed above, the pressure value $\up \in \rea_{\ge 0}$ can be adjusted in real-time and viewed as an additional input that changes the robotic dynamics. This paper aims to develop a control-oriented dynamic model for continuum robots with layer jamming. For control purposes, it is essential to have simplified models that capture the key characteristics of LJ-based continuum robots. In particular, when the pressure $\up =0$, the proposed model should reduce to the LJ-free model in Section \ref{sec:21}.

Intuitively, when adjusting of the vacuum pressure $u_{\tt P}$, it is reasonable to expect that some functions in the model depends on the pressure $u_{\tt P}$. To explore such a relationship, firstly we discuss the dependence of the plant parameters $\alpha_i$ ($i=1,2$) and the function $D(\cdot)$ on $u_{\tt P}$. The following assumptions are introduced to formalize this dependence.

\begin{assumption} \label{assm:1}
During the variation process of $\up$, the continuum robot satisfies:
\begin{itemize}
\item[\bf (a)] The mass change of air in the airtight membrane is negligible. Hence, the gravitational parameter $\alpha_1$ is constant and thus \emph{independent} of the pressure $\up$.

\item[\bf (b)] The elastic coefficient $\alpha_2>0$ is a function of the pressure $\up$. For simplicity, we write it as $\alpha_2(\up)$.   
\end{itemize}
\end{assumption}

\begin{assumption}
\label{assm:D}
The energy dissipation of the robot is only derived from the lumped friction torque with the damping matrix $D(q) =0$.
\end{assumption}

\begin{remark}\rm
In Assumption \ref{assm:1}, it is reasonable to consider constant mass during inflation and deflation. Now, let us briefly illustrate the underlying reason for considering a pressure-dependent function $\alpha_2(\up)$ in the point (b). In \cite{NARetal}, the \emph{elastic modulus} of the jamming structure has been quantitatively demonstrated to exhibit distinct values in the vacuum-on and the vacuum-off states by a factor of $k_n^2$, where $k_n$ represents the number of layers. This implies that $U_{\tt E}$ is determined by the pressure $\up$ for a given configuration. Consequently, the elastic coefficient $\alpha_2$ in the proposed model should be a function of $\up$. This point can be experimentally validated, and the identification of the nonlinear function $\alpha_2(\up)$ will be studied in the subsequent sections.
\end{remark}

Under the above assumptions, the model \eqref{eq:model} of the LJ-based continuum robot can be compactly written as
\begin{equation}
\label{eq:s1}
\Sigma_{r}: \quad 
\left\{
\begin{aligned}
\dot x & = J  \nabla H(x) + G_{r}(x)u - G_{f} \tau_{f}
\\
v & = G_f^\top \nabla H(x)
\end{aligned}
\right.
\end{equation}
with the new variable $x:= \col(q,p)$ and 

\begin{equation*}
\begin{aligned}
    G_{r} = \begmat{~0_{n\times m}~\\ G(q)}, ~ 
    G_{f} = \begmat{~0_{n\times n}~\\ I_n},~  
    J = \begmat{~0_{n\times n} & I_n \\ - I_n  & 0_{n\times n}~},
\end{aligned}
\end{equation*}
where $\tau_{f} \in \rea^n$ is the lumped frictional torque acting in the links. If we view the friction vector $\tau_f$ as the ``input'', then the \emph{passive output} $v \in \rea^n$ is, indeed, the generalized velocity \cite{ORTetalBOOK,VAN}, i.e.,
\begin{equation}
\label{eq:v}
	v = M^{-1}(q) p.
\end{equation}

The jamming phenomenon, indeed, arises from the distributed friction along the layers, and the remaining task for modeling boils down to studying the frictional effects from $\tau_f$ and its interconnection to the system $\Sigma_r$.  

To take this behavior into account in the model, we adopt the LuGre friction model which was proposed in \cite{DEWetal}. The LuGre model is a \emph{dynamical} model that effectively describes various frictional properties, such as zero-slip displacement (a.k.a. micromotion), slick-slip motion, invariance, state boundedness, and passivity \cite{ASTDEW}. Later, we will demonstrate that this model is also suitable for the dynamic modeling of LJ-based continuum robots.

Before presenting our modeling approach, we make the following assumption about the (lumped) normal force $F_n >~0$ between the jamming layer surfaces.

\begin{assumption}
\label{assm:2}
The pressure along the jamming layer is uniformly distributed with a value $(-\up)$ and is proportional to some function of the lumped normal force, i.e., $F_n~\propto~\phi(\up)$, where $\phi: \rea^n \to \rea^n$ is yet to be determined. 
\end{assumption}

To facilitate the analysis in the subsequent sections, we adopt the port-Hamiltonian form of the LuGre model, as introduced in \cite{KOOetal}:
\begin{equation}
\label{eq:LuGre}
\Sigma_{z}: \quad
\left\{
\begin{aligned}
\dot z &= - R_{z}(v) \nabla H_{z}(z) + [\caln(v)-\calp(v)]v
\\
\tau_f & = [\caln(v)+\calp(v)]^\top \nabla H_z(z) + Sv,
\end{aligned}
\right.
\end{equation}
where $z \in \rea^n$ represents the virtual bristle deflection at each (virtual) joint, $v\in \rea^n$ is the input -- the relative generalized velocity of the surfaces in contact given by \eqref{eq:v}, and the output $\tau_f\in \rea^n$ is the frictional torque. The mappings in $\Sigma_z$ include the virtual bristle potential energy
\begin{equation}
\label{eq:Hz}
H_z(z) = \hal \sigma_0\phi(\up) |z|^2,
\end{equation}
the damping matrix 
\begin{equation}
\label{eq:Rz}
\begin{aligned}
R_z(v) & ~=~ \diag(\beta_1(v), \ldots, \beta_n(v))
\\
\beta_i(v) &~:=~ {|v_i|\over \phi(\up)} \rho(v_i),\quad i \in \bar n 
\end{aligned}
\end{equation}
the state-modulation matrices
\begin{equation}
\label{eq:mappings}
\begin{aligned}
	\caln(v) &~:=~  I_n - \hal \sigma_1 \phi(\up) R_z(v)
	\\
	\calp(v) &~:=~ - \hal \sigma_1 \phi(\up) R_z(v)
	\\
	S &~:=~ (\sigma_1 + \sigma_2) \phi(\up) I_n,
\end{aligned}
\end{equation}
and the function
\begin{equation}
\label{eq:rho}
\rho(v_i) = \mu_C + (\mu_S -\mu_C) \exp\left(-\left|{v_i\over v_s}\right|^{\sigma_3} \right).
\end{equation}
Physical meanings of coefficients in the above model are summarized in Table \ref{tab:coeff}. The interested reader may refer to \cite{DEWetal,BARORT,KOOetal} for additional details. The above provides an energy-based modeling approach to model the dynamical behaviors caused by frictions in the LJ-based continuum robot.

\begin{table}[!htp]
\begin{center}
\caption{List of Coefficients in the LuGre Model}
\label{tab:coeff}
\small
\begin{tabular}{r|l}
\hline
$\mu_S$ & Stiction friction coefficient
\\
$\mu_C$ & Coulomb friction
\\
$v_s$ & Stribeck velocity
\\
$\sigma_0$ & Bristle stiffness coefficient
\\
$\sigma_1$ & Bristle damping coefficient
\\
$\sigma_2$ & Viscous friction coefficient
\\
$\sigma_3$ & Curve parameter (further tuning of the Stribeck effects)
\\
\hline
\end{tabular}
\end{center}
\end{table}

Note that in the model presented above, there is a need to identify the unknown function $\phi$ in our continuum robot. From certain physical considerations, we have a priori knowledge of the properties of the function $\phi$, which are summarized in the following assumption.
\begin{assumption}
\label{assm:phi}
The function $\phi$ is smooth and satisfies:
\begin{itemize}
    \item[(a)] It is positive definite, i.e., $\phi(\up) \ge 0$ for all $\up\ge 0$.
    \item[(b)] The function is monotonically increasing with respect to $\up$. 
\end{itemize}
\end{assumption}
The function can be obtained either from mechanisms or data-driven methods. This will be further investigated in the subsequent sections. 

\begin{remark}\rm
The model $\Sigma_z$ is well-posed for all $\phi(\up) \ge~0$ even though $\phi(\up)$ appears in the denominator of the function $\beta_i$ in \eqref{eq:Rz}. This is due to the product $R_z(v) \nabla H_z$ in the dynamics and $\nabla H_z$ being linear in $\phi(\up)$. If the function $\phi(\up) =0$, we have $\tau_f =0$ for which we roughly regard zero friction injected to the robotic mechanical dynamics $\Sigma_r$. The friction torque $\tau_f$ at the steady-state stage becomes 
$$
\tau_f^{\tt ss} = [ \diag\{\rho(v_i)\}\sign(v) + \sigma_2v ]\up,
$$
with $\sign(v) := \col(\sign(v_1), \ldots, \sign(v_n))$ collecting the signs of $v_i$.
\end{remark}

\begin{remark}\rm
The LuGre model has the boundedness property for the internal state, i.e., the set $\cale_z:= \{z\in \rea^n: ~|z| \le {\mu_S \over \sigma_0}\}$ \cite{ASTDEW}. The input-output pair $(v, \tau_f)$ satisfies the particularly appealing \emph{passivity} property 
$$
\int_0^t v^\top(s) \tau_f(s) ds \ge   H_z(z(t)) - H_z(z(0)), \quad \forall t\ge0
$$
if the coefficients satisfy the inequality constraints \cite{BARORT}:
\begin{equation}
\label{sigma12}
    \sigma_2 > \sigma_1 {\mu_S - \mu_C \over \mu_C}.
\end{equation}
These properties are relevant in the context of LJ-based continuum robots.
\end{remark}

\subsection{Overall Dynamical Model for LJ-based Continuum Robots}
\label{sec:23}
In this subsection, we propose the overall dynamical model for LJ-based continuum robots, which is the negative interconnection of $\Sigma_r$ and $\Sigma_z$. 

For convenience, we define the full systems state as

\begin{equation}
    \chi:= \begmat{~q~ \\ ~p ~\\ ~z~} \in \rea^{3n}.
\end{equation}
Its dynamics can be compactly written in the port-Hamiltonian form
\begin{equation}
\label{eq:overall}
\dot\chi = [\calj - \calr] \nabla\calh + \calg(\chi)u
\end{equation}
with the total Hamiltonian 
\begin{equation}
\label{eq:calh}
\begin{aligned}
& \calh(\chi, \up)  ~:=~  H(q,p) + H_z(z, \up)
\\
 & ~=~ \underbrace{\hal p^\top M^{-1}(q) p}_{\mbox{kinematic energy}} +  \underbrace{\hal \sigma_0 \phi(\up) |z|^2 + U(q)}_{\mbox{potential energy}}
\end{aligned}
\end{equation}
and the matrices
\begin{equation}
\label{eq:ovjr}
\begin{aligned}
    \calj(\chi, \up) & ~:=~ \begmat{J & - G_f \caln^\top~ \vspace{0cm}
    \\ ~\caln G_f^\top & 0_{n\times n} } 
    \\
    \calr(\chi, \up) & ~:=~ \begmat{~G_f S(v) G_f^\top & G_f \calp^\top ~\vspace{0cm}
    \\ \calp^\top G_f^\top & R_z}
    \\
    \calg(\chi) &~:=~ \begmat{G_r \\ 0_{n\times m}}.
\end{aligned}
\end{equation}
Note that $\caln, \calp$ and $S$ are linear functions of the pressure $\up$. The overall model has an $(m+1)$-dimensional input 
$$
u_\chi = \begmat{~u~ \\ ~\up~}
$$
with all the elements non-negative. 

\begin{remark}
\label{rem:R_wlps}\rm
The damping matrix $\calr$ can be decomposed as $\calr = \diag(0_{n\times n}, \calr_{\tt pz})$ with 
\begin{equation*}
    \calr_{\tt pz} := \begmat{(\sigma_1 + \sigma_2)\phi(\up) I_n & -\hal \sigma_1 \phi(\up) R_z(v) \vspace{.2cm}
    \\ 
    -\hal \sigma_1 \phi(\up) R_z^\top(v) & R_z(v)}.
\end{equation*}
It is shown in \cite[Sec. 4]{KOOetal} that positive semi-definiteness of $\calr$ can be guaranteed under the coefficient constraint \eqref{sigma12}. This makes $\calr$ qualified as a \emph{damping} matrix. 
\end{remark}

\section{Interpretation to Key Phenomena in LJ-based Continuum Robots} \label{sec:3}
In this section, we theoretically verify that the model proposed in Section \ref{sec:23} can interpret the two phenomena in LJ-based continuum robots -- shape locking and adjustable stiffness.

\subsection{Phenomenon 1: Shape Locking} \label{sec:31}
Shape locking is one of the important capabilities of LJ structures when applied to continuum robots \cite{ZHAetalBB,NARetal,SANetal}. Tensions along the cables can change the robot's configuration from its undeformed shape; when a vacuum with negative pressure $(-\up)$ is applied before the release of tension actuation, the continuum robot is able to preserve its current shape. This phenomenon is known as shape locking. In this subsection, we aim to illustrate that shape locking can be characterized by the proposed model. 

Firstly, we mathematically formulate the shape locking behavior described above, as defined below.

\begin{definition}
\label{def:ShapeLock} (\textit{Shape Locking}) Consider the LJ-based continuum robotic model with zero input $u$. If the deformed configuration $\bar q \neq 0_3$ guarantees the set $\cale_{\tt SL}:=\{(q,p,z) \in \rea^{3n}: q= \bar q , ~ p = 0_3\}$ under $\up >0$ forward invariant, i.e.,
\begin{equation}
\chi(0) \in \cale_{\tt SL} ~ \implies ~ [~\dot q(t) = 0,~\dot p(t) =0, ~\forall t\ge 0~],
\end{equation}
then we call this invariance as shape locking.
\end{definition}

The definition above provides a mathematical characterization of the shape-locking phenomenon. Intuitively, under a certain negative pressure $(-\up)$, a continuum robot maintains a deformed configuration $\bar q \neq 0_3$ over time. This behavior can be experimentally observed in LJ-based continuum robots.

The following proposition gives a theoretical analysis of the shape-locking phenomenon using the proposed model. A more detailed illustration will follow immediately after its proof.

\begin{proposition}
\label{prop:ShapeLocking}
Consider the LJ-based continuum robot model \eqref{eq:overall} without external input, i.e., $u=0_m$. For arbitrary configuration $q_a \in \rea^n$ and a \emph{constant} pressure $\up>0$,
\begin{itemize}
    \item[\bf (a)] There exists a vector $z_a \in \rea^n$ such that $(q_a, 0_n, z_a)$ is an equilibrium;
    \item[\bf (b)] The equilibria manifold 
$$
 \calm := \{(q,p,z)\in \rea^{3n}: p =0, ~\nabla U(q) = \sigma_0 \phi(\up) z\}
$$
is locally asymptotically stable.
\end{itemize}
\end{proposition}

\begin{proof}
First, let us verify the existence of $z_a$ such that $(q_a, 0_n, z_a)$ is an equilibrium. From \eqref{eq:v}, $p=0$ implies the velocity $v=0$, thus
\begin{equation*}
    \dot q = \nabla_p H = M^{-1}(q)p =0.
\end{equation*}
The dynamics of $z$ is given by
\begin{equation*}
\begin{aligned}
    \dot z ~ =~ - R_z(0)\nabla H_z + [\caln - \calp]v \Big|_{v=0} ~=~ 0
\end{aligned}
\end{equation*}
where we have used the fact $R_z(0) = 0$ from \eqref{eq:Rz}. For the generalized momentum, we have the following
\begin{equation}
\begin{aligned}
     \dot p   ~=~& -  {\partial \over \partial q} \left\{\hal p^\top M^{-1}(q)p \right\} - \nabla U(q_a) + Sv \\ 
     & + [\caln+\calp]\nabla H_z \Big|_{p=0}
     \\
      ~=~& -\nabla U(q_a) + [I_n-\sigma_1 \phi R_z(v)]\sigma_0 \phi z
     \\
      ~=~ & -\nabla U(q_a) + \left(I_n - \sigma_1 \diag\left\{{|v_i|\over \rho(v_i)} \right\} \right)\bigg|_{v=0}\sigma_0 \phi z
      \\
     ~=~ & - \nabla U(q_a) + \sigma_0 \phi z,
\end{aligned}
\end{equation}
in which we have simply written $\phi(\up)$ as $\phi$. Invoking Assumption \ref{assm:phi}, for any non-zero $\up$, we conclude that the point 
$$
\chi_\star := \begmat{~q_a~ \\ ~0_n~ \\ ~z_a~}
$$
with
\begin{equation}
    z_a:= {1\over \sigma_0 \phi(\up)}\nabla U(q_a)
\end{equation}
is an equilibrium. 

The next step of the proof is to show the local asymptotic stability of the manifold $\calm$. Calculating the time derivative of the overall Hamiltonian $\calh$, it yields for $\chi \in B_\varepsilon(\calm)$ with a small $\varepsilon>0$, 
\begin{equation}
\begin{aligned}
    \dot\calh & ~=~ - [\nabla\calh(\chi,\up)]^\top \calr(\chi,\up) \nabla \calh(\chi,\up)
    \\
    & ~\le~ - \left\|\begmat{\nabla_p \calh \\\nabla_z \calh}\right\|_{\calr_{22}}^2
    \\
    & ~\le~ 0,
\end{aligned}
\end{equation}
in which we have used the fact that in $B_\varepsilon(\calm)$ the matrix $\calr$ is positive semidefinite from Remark \ref{rem:R_wlps}. Thus, in the neighborhood of the manifold $\calm$, the system is Lyapunov stable. In the set 
\begin{equation}
\{\chi\in\rea^{3n}: \|\col(\nabla_p \calh, \nabla_z \calh)\|_{\calr_{22}} = 0\},
\end{equation}
we need to verify
\begin{align}
(\sigma_1 + \sigma_2) M^{-1}(q)p - \hal \sigma_0 \sigma_1 \phi(\up) R_z(v) z & ~=~ 0
\label{eq:a1}\\
-\hal\sigma_1 R_z(v) M^{-1}(q)p + \sigma_0 R_z(v) z & ~=~0. \label{eq:a2}
\end{align}
Let's first consider \eqref{eq:a2}. There are two possible cases:
\begin{itemize}
    \item case (i): $R_z(v) =0$ (or equivalently $p =0$).
    \item case (ii): For some $j\in \bar n$, $\beta_j(v) \neq 0$, and thus
    \begin{equation}
    \label{eq:M-1p}
        M^{-1}(q) p = 2 {\sigma_0 \over \sigma_1}z.
    \end{equation}
\end{itemize}
For case (i), the trajectory verifies $p(t) \equiv 0$, thus 
\begin{equation*}
    \dot p = - \nabla U(q) + \sigma_0 \phi(\up) z =0,
\end{equation*}
which is exactly the manifold $\calm$. For case (ii), we substitute \eqref{eq:M-1p} into \eqref{eq:a1}, resulting in 
\begin{equation}
\label{eq:a3}
   4(\sigma_1 + \sigma_2) z  =  \sigma_1^2\beta_j(v) \phi(\up) z.
\end{equation}
There are two sub-cases: case (ii-1) $z=0$ and case (ii-2) $z\neq 0$. For case (ii-1), the trajectory should guarantee $z\equiv 0$ and thus 
\begin{equation*}
\begin{aligned}
    \dot{z} &~=~ - \calr_z(v) \nabla H_z(0) + [\caln - \calp ]v \Big|_{v\neq 0} 
    \\
    &~=~ [\caln - \calp ]v \Big|_{v\neq 0} =0.
\end{aligned}
\end{equation*}
Since $\caln - \calp = I_n$, it contradicts with $v \neq 0$ in case (ii). Thus, there is no feasible trajectory. For case (ii-2), the equation \eqref{eq:a3} can be rewritten as
\begin{equation}
\label{eq:sigma}
       \sigma_1 + \sigma_2  =  {\sigma_1^2 |v_j|\over 4\rho(v_j)}.
\end{equation}
Note that $\lim_{|v|\to 0}\rho(v_j) = \mu_C$. For given coefficients $\sigma_1, \sigma_2$, \eqref{eq:sigma} does not admit any feasible solution for a sufficiently small $\varepsilon>0$.
Therefore, all the feasible solutions in $B_\varepsilon(\calm)$ are all on the equilibria manifold $\calm$.

The system is autonomous since we consider constant pressure $\up$. As we have shown above, $\calm$ is the largest invariant set in the neighborhood $B_\varepsilon(\calm) \subset \rea^{3n}$. Since the overall dynamics is autonomous, applying the LaSalle's invariance principle \cite[Sec. 3]{KHAbook} we have that the manifold $\calm$ is locally asymptotically stable.  
\end{proof}


\begin{remark}
\label{rem:long}\rm
The above proposition shows that 
\begin{itemize}
    \item[(i)] If the initial condition $\chi(0)$ starts from any configuration $q_a$ and zero momentum $p(0)=0$, we may always find a virtual bristle vector $z_a$ such that the continuum robot's configuration maintains at the initial values over time, and we also note $\calm \subset \cale_{\tt SL}$. In this way, it achieves shape locking. 
    \item[(ii)] A more realistic scenario is that the continuum robot achieves deformation with the tension input $u \in \rea^m$; then we apply a vacuum and release the actuator. Once completing the tension release, the initial condition is given by $\chi(0) =(q(0), 0_3, 0_3)$ instead of $(q_a, 0_3, z_a)$. Proposition \ref{prop:ShapeLocking}(b) shows the \emph{local} asymptotic stability of the manifold $\calm$, which means if the initial distance
    \begin{equation}
        d(\chi(0), \calm) := \inf_{\chi' \in \calm } |\chi' - \chi(0)| < \varepsilon_0
    \end{equation}
    is sufficiently small, the trajectory ultimately converges to equilibrium $(q_a, 0_3, z_a) \in \calm$.
    
    \item[(iii)] From (ii), the convergence only happens when \mbox{$\varepsilon_0>0$} is sufficiently small. Note that the vector $z_a$ is parameterized as $z_a = {1\over \sigma_0 \phi(\up)}\nabla U(q_a)$. Thus, a large value of $\up$ can impose the initial condition $(q(0), 0_3, 0_3)$ in a small neighborhood of $\calm$; see Fig.~\ref{fig:manifold} for an intuitive illustration. \emph{Physically, it means that a large pressure value $\up$ is capable of achieving shape locking.}

    \item[(iv)] The above item shows that after releasing the actuation, the system will change from the initial configuration $(q(0), 0_3, 0_3)$ to the new equilibrium $(q_a, 0_3, z_a)$, and it will be closed to each other with a high pressure $\up$. It means when the continuum robot changes from flexible to stiff, we may observe a \emph{tiny} positional change that has been experimentally verified in \cite[Sec. III-B]{CLAROJ}. In Section \ref{sec:5}, we further support this point with experimental validations.
    \begin{figure}[!htp]
        \centering
        \includegraphics[width = 0.85\linewidth]{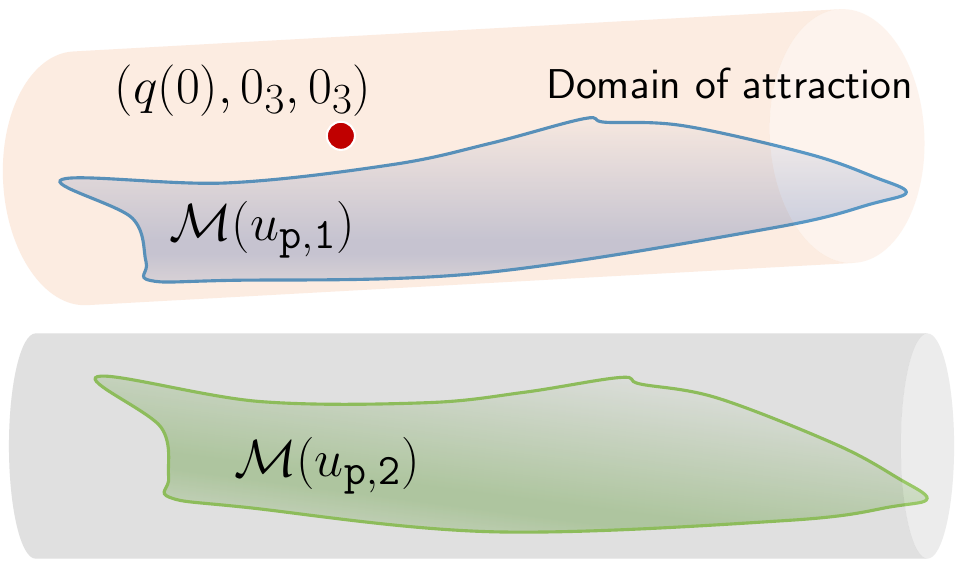}
        \caption{An illustration of the initial condition and the equilibria manifold $\calm$. For a given initial condition $(q(0), 0_3,0_3)$, a larger $u_{\tt p,1}$ implies a smaller distance from $\chi(0)$ to $\calm$, thus $\chi(0)$ located in its domain of attraction; a smaller $u_{\tt p,2}$ may cause the initial condition outside the domain of attraction, failing to achieve shape locking.}
        \label{fig:manifold}
    \end{figure}
\end{itemize}
\end{remark}

\subsection{Phenomenon 2: Adjustable Open-loop Stiffness} \label{sec:32}
When there is no external control input, the open-loop equilibrium $\chi_\star:=(q_\star, p_\star, z_\star)$ for the overall dynamics is the origin. For the stiffness analysis, we assume that there is an external torque $\tau_{\tt ext}$ acting on the dynamics of $p$, i.e., the dynamics with $u=0$ becomes 
\begin{equation}
\label{dot_bar_chi}
 \dot {\bar \chi} = [\calj - \calr] \nabla \calh + G_0 \tau_{\tt ext}.
\end{equation}
with 
$
G_0= \begmat{~0_{3\times 3} ~ ~I_3~ ~0_{3\times 3}~}^\top,
$
under which there is a shifted equilibrium $\bar\chi:=\col(\bar q , 0, \bar z)$.

Throughout the paper, we consider the stiffness in the following sense.
\begin{definition}
\label{def:stiff}  (\textit{Stiffness})
Assume that we can find a positive semidefinite $K \in \rea^{3\times 3}_{\succeq0}$ such that
\begin{equation}
\label{tau_ext}
    \tau_{\tt ext} := K(\bar q - q_\star) 
\end{equation}
solves \eqref{dot_bar_chi}-\eqref{tau_ext}. When taking $\bar q \to q_\star$ and $\bar z \to z_\star$, if the limit of $K$ exists, we call $K$ the overall stiffness.
\end{definition}

Another key phenomenon in LJ-based continuum robots is the ability to tune the open-loop stiffness over a wide range by adjusting the negative pressure $(-\up)$. The following proposition demonstrates that the proposed dynamical model effectively captures this open-loop stiffness adjustability in LJ-based continuum robots.

\begin{proposition}
\label{prop:stiffness}
Consider the LJ-based continuum robotic model \eqref{eq:overall}. Its overall stiffness in the sense of Definition \ref{def:stiff} at the open-loop equilibrium $\chi_\star$ is given by
\begin{equation}
\label{eq:K}
\begin{aligned}
K = \alpha_1 \mathbf{1}_{n\times n} + [\alpha_2 + \sigma_0 \phi(\up) ]I_n,
\end{aligned}
\end{equation}
with $\mathbf{1}_{n\times n} \in \rea^{n\times n}$ denoting an all-ones matrix.
\end{proposition}
\begin{proof}
Let us consider a tiny displacement $(\delta q, \delta z)\in \rea^n~\times ~\rea^n$ around $(q_\star, z_\star)$, i.e.
\begin{equation}
\label{eq:dqdz}
\begin{aligned}
q &= q_\star + \delta q \\
z &= z_\star + \delta z.
\end{aligned}
\end{equation}
For ease of analysis, we rewrite the model in an Euler-Lagrangian form 
    \begin{align}
         M(q) \ddot q + C(q,\dot q) \dot q + \nabla U(q) & ~=~ \tau_{\tt ext} - \tau_f \label{eq:EL}
        \\
        \dot z & ~=~ - \sigma_0 \diag \left\{ { |\dot q_i| \over \rho (\dot q_i)} \right\} z + \dot q \nonumber
        \\
        \tau_f & ~=~ (\sigma_0 z + \sigma_1 \dot z + \sigma_2 \dot q) \phi(\up), \nonumber
    \end{align}
with zero initial condition, in which $C(q,\dot q)$ is the Coriolis and Centrifugal term \cite{ORTetalBOOK}. 

Linearizing the dynamics \eqref{eq:EL} around $q_\star =0, \dot q_\star =0$ and $z_\star =0$ and invoking \eqref{eq:dqdz}, we obtain the model
\begin{equation}
\label{eq:linear}
\begin{aligned}
    M_\star \delta \ddot q + [\sigma_1\phi(\up)] \delta \dot q +  [\nabla^2 U(q_\star) &+ \sigma_0\phi(\up)I_3]\delta q \\
    & = \tau_{\tt ext} + \mathcal{O}(\delta q^2)
\end{aligned}
\end{equation}
with 
$$ M_\star := M(q_\star) $$
and high-order remainder term $\calo(\delta q^2)$, in which we have used the facts 
$$
C(q_\star, 0) =0, \quad \nabla U(q_\star) = 0. 
$$
Since 
$$
\begin{aligned}
    \sigma_1 \phi(\up) & ~>~ 0
    \\
    \nabla^2U(q_\star) + \sigma_0 \phi(\up) I_3 & ~\succ~ 0,
\end{aligned}
$$
the linearized dynamics \eqref{eq:linear} is exponentially stable at equilibrium
$$
\delta q = [\nabla^2U(q_\star) + \sigma_0 \phi(\up) I_3 ]^{-1} \tau_{\tt ext} + \mathcal{O}(\delta q^2).
$$
By taking $|\delta q|\to0$, the algebraic equation \eqref{tau_ext} is obtained with $K$ given 
by 
$$
K = \nabla^2 U(q_\star) + \sigma_0 \phi(\up) I_3.
$$
Substituting the function $U$ in \eqref{eq:U} into the above, we obtain \eqref{eq:K} and complete the proof. 
\end{proof}

The above result shows that in the open-loop case, increasing the vacuum pressure value $\up>0$ leads to a corresponding increase in stiffness, with their algebraic relation is given in \eqref{eq:K}. This result is compatible with existing static modeling findings.

\section{Feedback Control}
\label{sec:4}
In this section, we aim to design a feedback controller capable of simultaneously regulating the configuration and stiffness of LJ-based continuum robots in real time.

\subsection{Assumptions and Preliminaries}
Before presenting the controller synthesis, we make some key assumptions that will be used in the subsequent analysis. 

\begin{assumption}
The continuum robot has a piecewise constant curvature, and its actuator dynamics are negligible compared to the robot's overall dynamics. 
\end{assumption}

The above means that the motors in the robot are operated in torque control mode and we have neglected their dynamics. We also suppose that the tensions are distributed uniformly along the cables.

\begin{assumption}
\label{assm:uniform}
(\emph{Uniformity}) The input matrix $G(q)$ can be reparamterized as 
\begin{equation}
\label{eq:G}
G(q) ~= ~ g(q) \mathbf{1}_n
\end{equation}
with a $C^1$-continuous scalar function $g(q)$ and $|g(q)|$ is uniformly lower bounded from zero.
\end{assumption}

For underactuated mechanical systems, it is a well-known fact that only a specific class of desired configurations can be made as equilibria through feedback. These are commonly referred to as ``assignable equilibria.'' In our previous work \cite{YIetal23}, we have shown that under Assumption \ref{assm:uniform}, the \emph{homogeneous equilibria} belongs to the set of assignable equilibria for the
proposed continuum robotic model. They can be depicted by the set

\begin{equation}
\cale_\theta: = \{q\in \rea^n : q_i =\theta ,~ \forall i \in \bar n\}
\end{equation}
in which $\theta$ is a constant scalar, and this follows a slight modification of \cite[Proposition 3]{YIetal23}. In the sequel of this section, we will focus on stabilizing a desired configuration $q_\star \in \cale_\theta$.

To facilitate the controller design, we introduce the following \emph{input transformation}:
\begin{equation}
    \tau = T_u  u
\end{equation}
where $\tau\in \rea^2$ is the new input vector and the matrix $T_u$ is given by
\begin{equation}
    T_u := \begmat{1 & -1 \\ 0 & 1}.
\end{equation}
A benefit of the above-mentioned transformation is to avoid the sign constraint for $\tau_1$; see \cite[Sec. IV-A]{YIetal23} for additional details. Then, the dynamics \eqref{eq:overall} becomes
\begin{equation}
\label{eq:overall_tau}
\dot\chi = [\calj - \calr] \nabla\calh + \bar\calg(\chi)\tau
\end{equation}
with 
$$
G_\tau:= G(q) T_u^{-1}
$$ 
and
\begin{equation}
    \bar\calg:= \begmat{~0_n~ \\ ~G_\tau(q) ~\\ 0_n~}.
\end{equation}
There are now three inputs, i.e., $\col(\tau, \up) \in \rea^3$, in the model. In the sequel of the paper, we will design a controller based on the transformed model \eqref{eq:overall_tau} to achieve decoupling control of position and stiffness.

\subsection{Decoupling Control of Position and Stiffness}
In this section, we propose a feedback controller to achieve the decoupling control of position and stiffness for LJ-based continuum robots. Our main design is summarized below, and followed by its proof.

\begin{proposition}
\label{prop:control}
Consider the LJ-based continuum robot model \eqref{eq:model} with the coefficients satisfying the constraint \eqref{sigma12} and $\phi(\up) >0$. For any assignable desired configuration $q_\star = \col(\theta_\star, \ldots, \theta_\star)$, the feedback law
\begin{equation}
\label{eq:ctrl1}
\begin{aligned}
    \tau & ~=~ \tau_{\tt es}(q) + \tau_{\tt da}(q)
    \\
    \up & ~=~ \up^\star 
\end{aligned}
\end{equation}
with the functions

\begin{equation}
\label{eq:tau}
\begin{aligned}
\tau_{\tt es} (q) &=
{1\over g(q)}\begmat{
\alpha_1 \sin(q_\Sigma) - \gamma \sin(q_\Sigma - q_\Sigma^\star) + \alpha_2 \theta_\star 
\\
0
},
\\
\tau_{\tt da}(q) &=
- \begmat{  K_{\tt D} G_\tau (q)^\top  \nabla_p H \\ 0}.
\end{aligned}
\end{equation}
and the gains $\gamma >0$ and $K_{\tt D} \succ 0$ guarantees that
\begin{itemize}
    \item (\emph{configuration regulation}) The desired configuration $q_\star$ is a globally asymptotically stable (GAS) equilibrium of the closed loop, such that the configuration $q(t)$ converges to its reference value $q_\star$ as $t\to\infty$, i.e.
\begin{equation}
\lim_{t\to \infty} |q(t) -q_\star| =0.
\end{equation}

    \item (\emph{adjustable stiffness}) Additionally, the closed-loop stiffness at $q_\star$ is given by
\begin{equation}
\label{eq:K2}
    K = \gamma \mathbf{1}_{n\times n} + [\alpha_2 + \sigma_0 \phi(\up^\star)] I_n
\end{equation}
for a constant pressure $\up^\star \ge 0$.
\end{itemize}

\end{proposition}
\begin{proof} 
First, we study the closed-loop stability of the robotic dynamics. The term $\tau_{\tt es}(q)$ is for energy shaping term, and it satisfies the following equation:
\begin{equation}
\label{eq:pde}
[\calj - \calr] \nabla \calh  + \bar \calg \tau_{\tt es}(q) ~=~  [\calj - \calr] \nabla \calh_d  
\end{equation}
with the shaped overall Hamiltonian
\begin{equation}
\label{Hd}
\calh_d(q,p,z) = \hal p^\top M^{-1}(q) p + \hal \sigma_0 \phi(\up) |z|^2 + U_{\tt d}(q),
\end{equation}
the desired potential energy function
\begin{equation}
\label{Ud}
U_{\tt d}(q) = - \gamma \cos(q_\Sigma - q_\Sigma^\star) + {1\over2}\alpha_2 |q - q_\star |^2,
\end{equation}
and the variable
\begin{equation}
    q_\Sigma^\star : = \sum_{i\in \bar n} q_{\star ,i }.
\end{equation}
Since the desired potential energy function $U_{\tt d}$ guarantees
\begin{equation}
\begin{aligned}
    \nabla U_{\tt d}(q_\star) & ~=~ 0 
    \\
    \nabla^2 U_{\tt d}(q_\star) &~ \succ~ 0,
\end{aligned}
\end{equation}
the function $U_{\tt d}$ is locally convex around $q_\star$.\footnote{Since we only shape the closed-loop potential energy function, the proposed design can be viewed as a potential energy shaping passivity-based controller \cite{ORTetal01}.}

By substituting $\tau = \tau_{\tt es} + \tau_{\tt da}$ and the specific expression of $\tau_{\tt es}$ into the open-loop model \eqref{eq:overall_tau}, the closed-loop dynamics now becomes
\begin{equation}
\label{eq:dyn_shaped}
\dot \chi = [\calj - \calr] \nabla \calh_d + \bar \calg(\chi)  \tau_{\tt da}.
\end{equation}
Equivalently, it can be rewritten as
\begin{equation}
\label{eq:cldyn}
\begin{aligned}
\begmat{\dot q \\ \dot p \\ \dot z}
&= 
\begmat{ \begmat{0 & I_n \\ - I_n & - (\sigma_1 + \sigma_2)\phi I_n }  & - G_f \caln^\top - G_f \calp^\top ~
    \\ \caln G_f^\top - \calp G_f^\top & - R_z } ~
\\
& \qquad \qquad \qquad \qquad  \times \begmat{\nabla_q \calh_d \\ \nabla_p \calh_d \\ \nabla_z \calh_d}
+
\begmat{0_n \\ G_\tau(q) \tau_{\tt da} \\ 0_n }.
\end{aligned}
\end{equation}

Selecting the damping injection term in \eqref{eq:tau} and calculating the time derivative of the Hamiltonian $\calh_d$, we have
\begin{equation}
\begin{aligned}
    \dot{\calh}_d & 
= -(\nabla \calh_d)^\top 
    \begmat{    0  & 0 & 0 \\ 0 & R_p + R_{\tt da} & \calp^\top \\ 0 & \calp & R_z } 
  \nabla \calh_d
\\
&  = -
 \begmat{\nabla_p \calh_d \\ \nabla_z \calh_d}^\top 
 \left(
 \underbrace{\begmat{R_p & \calp^\top \\ \calp & R_z}}_{:=R_{\tt pz}}
 + \underbrace{\begmat{ R_{\tt da} & 0 \\ 0 &0} }_{:= {R}_{\tt D}}
 \right)
  \begmat{\nabla_p \calh_d \\ \nabla_z \calh_d} 
\\
& \le 0,
\end{aligned}
\end{equation}
with
\begin{equation}
\begin{aligned}
  R_p & ~:=~  (\sigma_1 + \sigma_2)\phi(\up) I_n 
  \\
  R_{\tt da} & ~:=~ G_\tau (q) K_{\tt D}  G_\tau^\top(q)  .
\end{aligned}
\end{equation}
Following the standard Lyapunov analysis, the system trajectories will ultimately converge to the set 
\begin{equation}
\label{eH}
   \cale_H:=  \left\{ (q,p,z) \in \rea^{3n} : (R_{\tt pz} + R_{\tt D}) \begmat{\nabla_p \calh_d \\ \nabla_z \calh_d}  =0\right\}.
\end{equation}
Invoking the positive definiteness of $\calr_{\tt pz}$ (see Remark \ref{rem:R_wlps}) and positive semidefiniteness of $\calr_{\tt D}$, we conclude that $(\calr_{\tt pz}~+~ \calr_{\tt D})$ is uniformly full rank. Therefore, all the trajectories within the invariant set satisfy 
\begin{equation}
\begin{aligned}
    \begmat{ \nabla_p \calh_d \\ \nabla_z \calh_z } = \begmat{M(q)^{-1} p \\ \sigma_0\phi(\up) z}  = 0.
\end{aligned}
\end{equation}
From the full rankness of $M(q)$ and the assumption \mbox{$\phi(\up)>0$}, we have
\begin{equation}
\label{eq:pz0}
  [\chi(t) \in \cale_H,~ \forall t\ge 0] \quad \implies \quad   \begmat{~p~ \\ ~z~} \equiv 0_{2n},
\end{equation}
where $\cale_H$ is defined in \eqref{eH}.
Then, it yields
\begin{equation}
\begin{aligned}
     \dot p   ~=~& -  {\partial \over \partial q} \left\{\hal p^\top M^{-1}(q)p \right\} - \nabla U_{\tt d}(q) + Sv \\ 
     & + [\caln+\calp]\nabla H_z \Big|_{p=0, z= 0}
     \\
     ~=~ & - \nabla U_{\tt d}(q) ,
\end{aligned}
\end{equation}
and thus
\begin{equation}
\begin{aligned}
 \eqref{eq:pz0} \quad  & \implies \quad   \nabla U_{\tt d}(q) \equiv 0
 \\
 & \implies \quad [\alpha _2\theta_\star - \gamma \sin(q_\Sigma - n\theta_\star) ] \mathbf{1}_n  = \alpha_2 q
 \\
 & \implies \quad  q = q_\star,
\end{aligned}
\end{equation}
in which we have considered the space $q_i \in \calx$. As a result, there only exists an isolated equilibrium 
$$
\chi_\star:= \begmat{~q_\star ~\\ ~0_n~ \\ ~0_n~}
$$
in the invariant set $\cale_H$. According to LaSalle's invariance principle \cite[Chapter 4]{KHAbook}, the closed-loop equilibrium $\chi_\star$ is globally asymptotically stable (GAS).

The reminder part is to verify the closed-loop stiffness, which resembles the proof of Proposition \ref{prop:stiffness}. We consider a tiny displacement $(\delta q, \delta z) \in \rea^n \times \rea^n$ in the neighborhood of $(q_\star, 0_n)$. Applying a small external torque $\tau_{\tt ext}$, the configuration will shift to $(q,z)$ such that
\begin{equation}
\begin{aligned}
    q &= q_\star + \delta q 
    \\
    z &= z_\star + \delta z,
\end{aligned}
\end{equation}
satisfying the Euler-Lagrangian equation
    \begin{align}
        M(q) \ddot q + C(q,\dot q) \dot q + \nabla U_{\tt d}(q) & ~=~ \tau_{\tt ext} - \tau_f \label{eq:EL2}
        \\
        \dot z & ~=~ - \sigma_0 \diag \left\{ { |\dot q_i| \over \rho (\dot q_i)} \right\} z + \dot q \nonumber
        \\
        \tau_f & ~=~ (\sigma_0 z + \sigma_1 \dot z + \sigma_2 \dot q) \phi(\up), \nonumber
    \end{align}
from the initial condition $(q_\star, 0_n)$. Similarly, we have the following around $(q_\star, 0_n, n_n)$ 
\begin{equation}
M_\star \delta \ddot q + \sigma_1 \phi \delta \dot q + [\nabla U_{\tt d}(q_\star) + \sigma_0 \phi I_3] \delta q = \tau_{\tt ext} + \mathcal{O}(\delta q^2).
\end{equation}
By substituting $\tau_{\tt ext} =  K\delta q$ and taking $|\delta q|\to 0$, we obtain the overall stiffness 
\begin{equation} \label{open_loop_eq}
\begin{aligned}
    K & ~=~ \nabla^2 U_{\tt d}(q_\star) + \sigma_0 \phi(\up^\star)I_3
    \\
    & ~=~ \alpha_1 \mathbf{1}_{n\times n} + [\alpha_2 + \sigma_0 \phi(\up) ]I_n.
\end{aligned}
\end{equation}
It completes the proof. 
\end{proof}

\subsection{Discussions}

In this subsection, we first discuss the approaches to identify the functions $\alpha_2(\cdot)$ and $\phi(\cdot)$, and then provide some discussions to the proposed controller.

1) {\em Selection of $\alpha_2(\up)$.} From the proof of Proposition \ref{prop:control}, the point $(q_\star, 0_n, 0_n)$ is a closed-loop equilibrium under the proposed feedback law. At this equilibrium, the steady-state input is given by
\begin{equation}
\label{ustar}
u_\star = {1\over g(q)}  T_u^{-1}  \begmat{ \alpha_1 \sin(q_\Sigma^\star) + \alpha_2(\up) \theta_\star \\ 0}.
\end{equation}

It shows the potential use of \emph{steady-state signals} for identifying both the parameter $\alpha_1$ and the nonlinear function $\alpha_2(\up)$ as well. For a given desired configuration \mbox{$q_\star \in \rea^n$}, the equation \eqref{ustar} means that the steady-state input $u_\star$ depends on pressure $\up$, and this property has been experimentally verified and will be presented in the next section, showing the rationale to use a pressure-dependent function $\alpha_2(\up)$ rather than a constant coefficient. 

We fixed $q_\star$ and drew the relation between $\up$ and $u_{\star,1}$, as shown in Fig. \ref{fig:rela}, which directly illustrates the nonlinearity of $\alpha_2$. Since $\alpha_1$ is constant, we conclude that $\alpha_2(\cdot)$ is a monotonically increasing function. In terms of the shape of the curve, we suggest using\footnote{In our experiments, we found that this class of polynomial functions provides sufficient accuracy. It would be interesting to explore more advanced dynamic learning tools to approximate these functions, such as the Koopman operator \cite{YIMAN}, kernel methods \cite{PILetal}, and Gaussian process regression \cite{WILetal}.}

\begin{equation} \label{alpha2}
    \alpha_2(\up) = c_1 \up^2+c_2 \up+ c_3
\end{equation}
with some unknown coefficients $c_i >0 ~(i=1,2,3)$ to fit this function.

\begin{figure}[htbp]
\centering
\begin{tikzpicture}
    \begin{axis}[
        width  = 0.8\linewidth,
        height = 0.4 \linewidth,
        xlabel={Vacuum pressure $\up$ [kPa]},
        ylabel={Input force $u_{\star,1}$ [N]},
        ymajorgrids=true,
        grid style=dashed,
        legend pos=north west,
        scale only axis,
        label style={font=\footnotesize},
        tick label style={font=\footnotesize},
        xtick={0,5,10,15,20,25,30,35,40},
        xticklabels={0,5,10,15,20,25,30,35,40},
        xmin=0,
        xmax=40,
        ymin=0,
        ymax=60,
        axis background/.style={fill=white},        
    ]
    \addplot[name path=upper5,draw=none] coordinates {
        (0, 8.5895 + 0.1541)
        (5, 17.9293 + 0.1602)
        (10, 21.0106 + 0.0357)
        (15, 22.4012 + 0.2441)
        (20, 23.1164 + 0.3860)
        (25, 22.9882 + 0.2443)
        (30, 23.2452 + 0.1472)
        (35, 23.2609 + 0.4710)
        (40, 23.3844 + 0.1699)
    };
    
    \addplot[name path=lower5,draw=none] coordinates {
        (0, 8.5895 - 0.1541)
        (5, 17.9293 - 0.1602)
        (10, 21.0106 - 0.0357)
        (15, 22.4012 - 0.2441)
        (20, 23.1164 - 0.3860)
        (25, 22.9882 - 0.2443)
        (30, 23.2452 - 0.1472)
        (35, 23.2609 - 0.4710)
        (40, 23.3844 - 0.1699)
    };
    
    \addplot[fill=blue!40,opacity=0.8] fill between[of=upper5 and lower5];
    
    \addplot[only marks, mark=x, mark options={blue!90}] coordinates {
        (0, 8.5895)
        (5, 17.9293)
        (10, 21.0106)
        (15, 22.4012)
        (20, 23.1164)
        (25, 22.9882)
        (30, 23.2452)
        (35, 23.2609)
        (40, 23.3844)
    };

    \addplot[name path=upper10,draw=none] coordinates {
        (0, 20.1735 + 1.3921)
        (5, 37.1697 + 0.3028)
        (10, 45.0706 + 0.8402)
        (15, 49.4509 + 0.4518)
        (20, 50.1903 + 1.3852)
        (25, 51.1226 + 0.3205)
        (30, 51.3704 + 0.8097)
        (35, 52.3334 + 0.1139)
        (40, 52.5439 + 0.8065)
    };
    \addlegendimage{empty legend};
    \addplot[name path=lower10,draw=none] coordinates {
        (0, 20.1735 - 1.3921)
        (5, 37.1697 - 0.3028)
        (10, 45.0706 - 0.8402)
        (15, 49.4509 - 0.4518)
        (20, 50.1903 - 1.3852)
        (25, 51.1226 - 0.3205)
        (30, 51.3704 - 0.8097)
        (35, 52.3334 - 0.1139)
        (40, 52.5439 - 0.8065)
    };
    
    \addplot[fill=red!20,opacity=0.5] fill between[of=upper10 and lower10];
    
    \addplot[only marks, mark=diamond, mark options={red!80}] coordinates {
        (0, 20.1735)
        (5, 37.1697)
        (10, 45.0706)
        (15, 49.4509)
        (20, 50.1903)
        (25, 51.1226)
        (30, 51.3704)
        (35, 52.3334)
        (40, 52.5439)
    };
    
    \node[anchor=north west, color=red] at (rel axis cs:0.6,0.65) {$\diamond$ 10 deg};
    \node[anchor=north west, color=blue] at (rel axis cs:0.3,0.65) {$\times$ 5 deg};
    \end{axis}
\end{tikzpicture}
\vspace{-5mm}
\caption{The relationship between $\up$ and $u_{\star,1}$ for two example bending angles, with $q_\star = 5$ deg and $10$ deg, is shown. The experiments were repeated three times at each vacuum pressure. The symbols ``$\times$'' and ``$\diamond$'' indicate the mean values and the error bands represent $\pm$1 standard deviation.}
\label{fig:rela}
\end{figure}

\vspace{.3em}

2) {\em Modelling of $\phi(\up)$ via mechanisms.} The term $\phi(\up)$ is another unknown function that appears in our proposed dynamic model. Instead of a data-driven approach to identify this nonlinear function, we note that this nonlinearity from friction present between layers of thin material has been studied via underlying mechanisms in LJ-based continuum robots. In particular, it is shown in \cite[Eq. (21)]{KIMetal} that the resisting torque caused by membrane elongation and shear $\tau_m$ satisfies $\tau_m \propto \sqrt{\up}$. Based on this fact, we consider using the following 
\begin{equation}
\label{parameterization:phi}
    \phi(\up) = c_4 + c_5 \sqrt{\up}
\end{equation}
with yet-to-be-determined parameters $c_i>0$ ($i=4,5$) to fit the nonlinear function $\phi$. The experimental results in the next section will validate that this selection provides good fitting.

3) {\em Identification of coefficients $c_i$.} With above the parameterization to $\alpha_2$ and $\phi$, we can identify the five coefficients $c_i$ using experimental data. First, we regulated the continuum robot to different closed-loop equilibria ($q_\star$)s under different pressures. The recorded data are denoted as $(u_\star^j, \up^j, q_\star^j)$ with the superscript $j$ representing the $j$-th experiment that we have done and $j \in W :=\{1,\ldots, w\}, ~w\in\mathbb{N}_+$. The parameters $\alpha_1$ and $c_i$ ($i=1,2,3$) can be estimated by solving the optimization problem
\begin{equation}
\label{eq:opti}
\begin{aligned}
        \underset{\alpha_1,c_1,c_2,c_3}{\arg\min} &\quad  \sum_{j \in W}   J(q_\star^j, u_\star^j, \up^j)
        \\
   \mbox{s.t.} & \quad  \alpha_1, c_2, c_3 >0
\end{aligned}
\end{equation}
with the cost function
\begin{equation*}
\begin{aligned}
   J(q_\star^j, u_\star^j, \up^j)
   := \Big| T_u g(q_\star^j) u_\star^j - [\alpha_1\sin(q_\Sigma^{\star,j}) + \alpha_2(\up)q_{1,\star}^j] \Big|^2
\end{aligned}.
\end{equation*}
In general, we require a sufficient number of experiments, i.e., a relatively large value for $w$, to ensure that the identified parameters closely reflect their real values. The identification of $c_4$ and $c_5$ can be done in a similar way using the parameterization~\eqref{parameterization:phi} and the open-loop stiffness relation \eqref{eq:K}, which contains the unknown function $\phi$. 


4) {\em Discussions.} Below, we have some discussions about the proposed controller.

\begin{remark}\rm
In the proposed controller, there are three parameters: $\gamma, K_{\tt D}$ and $\up^\star$.
Parameters $\gamma$ and $K_{\tt D}$ appear in the energy shaping part, implying their direct correlation with the closed-loop performance for position regulation. Equation \eqref{eq:K} indicates that the overall stiffness $K$ is determined by $\gamma$ and $\up^\star$. The proposed feedback controller achieves simultaneous position and stiffness control. The experimental results presented in the subsequent section will demonstrate that the term $\sigma_0 \phi(\up^\star) I_n$ \emph{dominates} the stiffness $K$ in \eqref{eq:K}. As a result, we can reasonably assert the limited impact of $\gamma$ on the stiffness. This rationale underlies why we refer to it as ``decoupling position and stiffness control.''
\end{remark}

\begin{remark}\rm
In previous work, we addressed the control problem for tendon-driven continuum robots \emph{without} layer jamming. In fact, the tension term $\tau$ in the feedback law \eqref{eq:ctrl1} is capable of handling fractions caused by jamming. The new model proposed in this paper includes an additional input channel $\up$, providing the possibility of achieving decoupled stiffness control.
\end{remark}

\begin{remark}\rm
In \cite{GANetal}, the authors consider the problem of joint control for a class of underactuated port-Hamiltonian models with the LuGre dynamics. They specifically consider the case where the LuGre model includes the mechanical input matrix $G(q)$, which is instrumental in solving the matching equation in passivity-based control (PBC). In contrast, our paper explores a more general scenario in which $G(q)$ does not appear in the fractional $z$-dynamics, making our model more closely aligned with the situations encountered in continuum robots.
\end{remark}

5) {\em Different jamming sheaths and vacuum pressures.} In the shape-locking experiments, which will be introduced in the next section, we use a jamming sheath with five layers and high vacuum pressures with a maximum value of 80~kPa to show the shape-locking phenomenon clearly. On the other hand, in order to be able to balance the adjustable stiffness and configuration control performance---under a relatively large range of vacuum pressures---in the second group of experiments, we considered the jamming sheath with \emph{two} layers and the maximum vacuum pressure of 40~kPa to avoid that the robot is controlled into shape locking which will result in failed stiffness control. Additional details about the jamming sheath and the jamming flap pattern can be found in \cite{FANetal, KIMetal}.

\section{Experiments and Results}
\label{sec:5}

\subsection{Test Rig}

\begin{figure*}[ht]
    \centering
    \includegraphics[width = .9\linewidth]{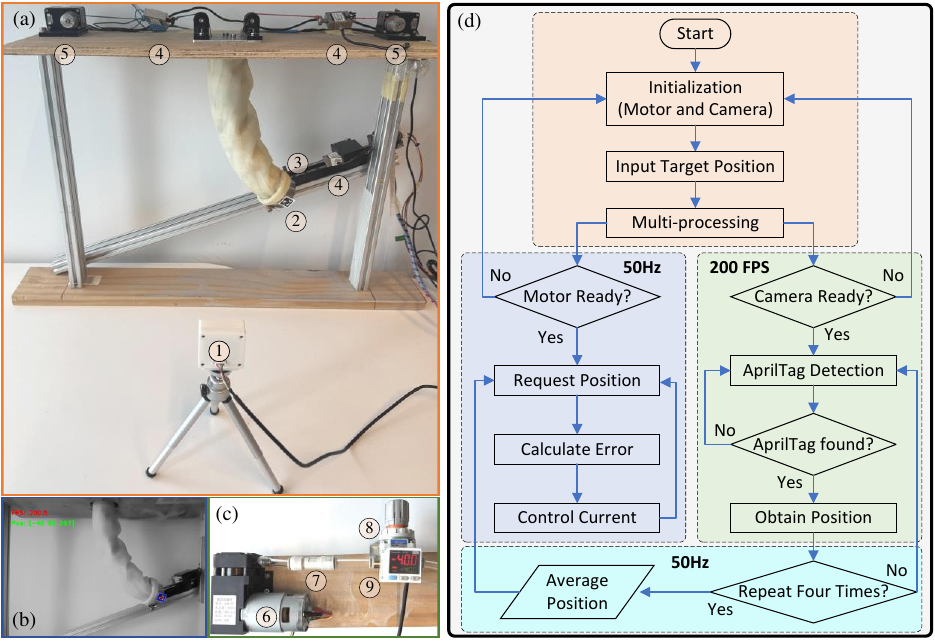}
    \caption{Photo of the entire experimental platform: (a) test rig setup; (b) camera view; (c) vacuum pressure generation setup; and (d) control process flowchart. (1 camera, 2 AprilTag, 3 linear actuator, 4 force sensor, 5 servo motor, 6 vacuum pump, 7 non-return valve, 8 vacuum pressure regulator, and 9 pressure gauge.)} 
    \label{fig:TestRig}
\end{figure*}

The proposed modeling and control approaches have been experimentally tested using the continuum robot \mbox{OctRobot-I} \cite{FANetal}. We considered six segments \mbox{(\emph{i.e.} $n=6$)} with an overall length of 252 mm, and a diameter of approximately 50 mm. The continuum robot approximately meets all the assumptions outlined in Sections \ref{sec:3}-\ref{sec:4}. 

The test platform, shown in Fig. \ref{fig:TestRig}(a), consists of the one-section continuum robot, two servo motors (XM430-W350, DYNAMIXEL) with customized aluminum spools, three force sensors (JLBS-M2-10kg), a linear actuator, and a custom high frame rate camera (MJEG-640*400@210FPS) for capturing the AprilTag marker \cite{apriltag}. The control experiments and data collections were conducted using the Python 3.12.3 environment. We installed the AprilTag marker on the distal position of the robot body to provide its real-time coordinate; see Fig.~\ref{fig:TestRig}(b). For the vacuum pressure, a micro piston vacuum pump (H40-85), vacuum pressure regulator (SMC IRV10-C08), and pressure gauge (Panasonic DP-100) are applied to achieve the desired vacuum pressure values, as illustrated in Fig.~\ref{fig:TestRig}(c).
We illustrate the overall control implementation process in Fig.~\ref{fig:TestRig}(d), which utilizes Python’s multiprocessing capabilities to enable simultaneous motor control and position detection.
To improve the accuracy of position capture and synchronize with the motor control frequency, we detected the AprilTag marker four times and calculated the average position data for control.\footnote{Note that multiple AprilTag markers could be attached to the robot body to provide additional information of the robot's configuration. However, this may introduce detection latency and make it challenging to capture all markers within the same frame due to noise and blur, especially when the robot is in motion.}

For configuration regulation, in the absence of external loads, we observed that the robotic system approximately satisfied the constant curvature condition, i.e. $q_i = q_j ~(i,j \in~\caln)$. Based on this observation and the coordinate along with basic geometric relations, the configuration vector $q \in \rea^n$ can be estimated in real-time by assuming \mbox{$q_i= \theta(t)$}.\footnote{Note that in theoretical analysis in Section \ref{sec:4}, we do \emph{not} assume $q_i = q_j ~(i,j\in \caln)$.} The notation $\theta$ will be used to represent the estimated value of $q_i$ throughout the remainder of this section. Unless otherwise specified, angles are expressed in deg and pressure in kPa.

The servo motors were operated in the current regulation mode such that we were able to control the torque directly for the robot. The experimental platform was equipped with force sensors, which were instrumental in analyzing the relation between bending angles and applied forces. However, in the closed-loop experiments, we removed the force sensors and regulated the motor torque directly.

\subsection{Model Verification for Shape Locking}
The first group of experiments was designed and conducted to verify that the proposed model effectively captures the key phenomenon of shape locking, as theoretically analyzed in Section \ref{sec:31}. The robot was initialized from the open-loop configuration $q_\star = 0$ (Phase 1), and then driven to the total bending angle of 60 deg via tendon (Phase 2). When the system kept at the steady-state stage, we vacuumed and kept the jamming layer sheath to a negative pressure of -30~kPa (Phase 3), and then released the tendons (Phase 4). 

During the above-mentioned process, the sequence photos are presented in Figs.~\ref{fig:PhotoSeq}(a)-(d), and the force of $u_1$ was operated as the trajectory in Fig.~\ref{fig:PhotoSeq}(e). Note that we use $[-5,0]$ s to denote the initial status before starting the motor drive. As illustrated in Figs.~\ref{fig:PhotoSeq}(c)-(d), it achieved shape locking after applying a negative pressure ($-\up$). To clearly show the shape-locking phenomenon, Fig.~\ref{fig:Photo2} illustrates the overlay photos of Phases 3 and 4 with two different pressures $\up$ = 30 and 80 kPa. It can be observed tiny positional changes as theoretically predicted in Remark~\ref{rem:long}(iv) -- a larger $\up$ yields a smaller displacement (3.8 mm of 80 kPa, 6.7 mm of 60 kPa, and 9.2 mm of 30 kPa).

\begin{figure}[!htp]
    \centering
    \includegraphics[width = 0.98\linewidth]{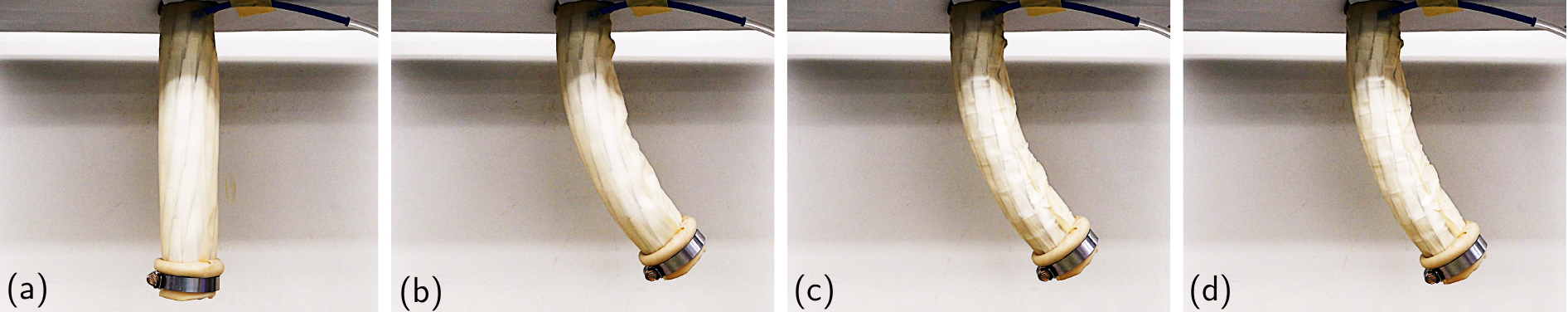}
    \par\vspace{1mm} 
    \includegraphics[width = 0.98\linewidth, height =3.5cm]{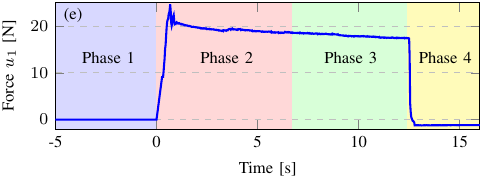}
    \caption{Photo sequence and force of the shape-locking experiment: 
    (a) Phase 1: Initial configuration without $u_1$; 
    (b) Phase 2: Drive to the bending configuration 60 deg via tendon force $u_1$;
    (c) Phase 3: Vacuum to $u_{\tt P} = 30$ kPa with motor-driven retained;
    (d) Phase 4: Vacuum retained and tendon released $u_1=0$.
    (Photos were taken in the steady state of each phase.)
    (e) Trajectory of the tendon force $u_1$.}
    \label{fig:PhotoSeq}
\end{figure}

\begin{figure}[!htp]
    \centering
    \includegraphics[width = 0.95\linewidth]{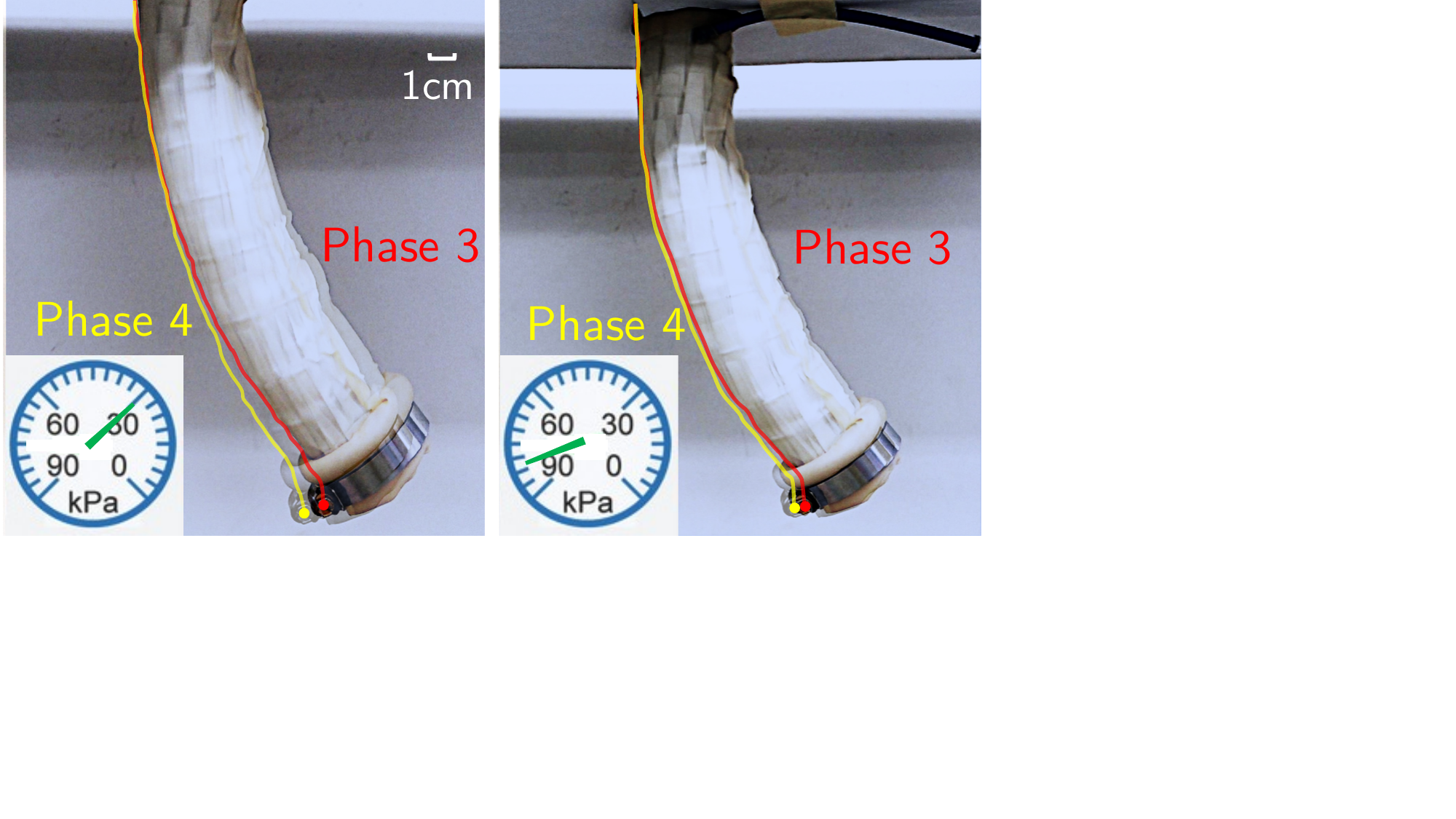}
    \caption{Overlay photos of the shape-locking phenomenon. Left: displacement of 9.2~mm with $\up$ = 30~kPa; Right: displacement of 3.8~mm with $\up$ = 80~kPa. (``\red{\tiny $\bullet$}'' and ``\textcolor{yellow}{\tiny$\bullet$}'' are used to mark a fixed point on the robot body; one-side contours are also highlighted in the figure.) 
    }
    \label{fig:Photo2}
\end{figure}

\subsection{Open-loop Stiffening}
In this section, we experimentally verify the results in Section~\ref{sec:32} on the continuum robotic platform OctRobot-I, i.e., the algebraic relation between the tension and the open-loop stiffness in Proposition \ref{prop:stiffness}. Although the overall stiffness matrix $K \succ 0$ cannot be measured directly, we can detect the transverse stiffness $K_{\tt T} \in \rea_{\ge 0}$ in the end-effector around the open-loop equilibrium $q_\star$ of the continuum robot. For this purpose, we utilized a linear actuator placed at the end-effector to generate a small displacement $\delta x > 0$, as illustrated in Fig. \ref{fig:OpenStiff}(a). The actuator was connected to the force sensor for measuring the external force, denoted as $f_{\tt ext}$, in relation to the displacement. By calculating the ratio of the measured force to the applied displacement, \emph{i.e.}, ${f_{\tt ext} / \delta x}$, we may estimate the transverse stiffness, given that $\delta x$ was sufficiently small. This procedure can be used to verify the theoretical findings in Section \ref{sec:3} regarding open-loop stiffening, thereby demonstrating the rationale of the proposed dynamical model for LJ-based continuum robots.

To improve reliability, each experiment has been repeatedly conducted three times under the same conditions. The experimental results of the open-loop transverse stiffness under different negative pressures $(-\up)$ are plotted in Fig.~\ref{fig:OpenStiff}(b). This clearly coincides with the claim in Proposition~\ref{prop:stiffness}. The coefficient of determination $R_s^2$ is 0.9216, showing good linearity with respect to the value $\sqrt{\up}$. The result aligns with the equation \eqref{eq:K}.

\begin{figure}[htbp]
    \centering
    \setlength{\tabcolsep}{1pt} 
    \begin{tabular}{c c}
        \includegraphics[width = 0.3\linewidth]{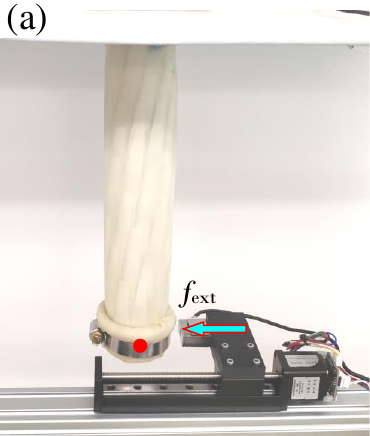} & \begin{tikzpicture}
    \begin{axis}[
      xmin=0, xmax=40,
      ymin=0, ymax=350,
      ymajorgrids=true,
      grid style=dashed,
      label style={font=\footnotesize},
      tick label style={font=\footnotesize},
      legend style={
        draw=none,
        fill=none,   
        font=\footnotesize,   
        at={(0.28,0.95)},     
        anchor=north,         
        cells={anchor=west}   
      },
      width  = 0.73 \linewidth,
      height = 0.42 \linewidth,
      xlabel={Pressure $\up$ [kPa]},
      ylabel={Stiffness [N/m]},
      ylabel style={yshift=-10pt}
    ]
    
  \addplot[domain = 0:40, color=babyblue, line width = 2pt] {29.1067*sqrt(x)  + 110.9847};
  \addlegendentry{$K_{\tt T}$}
  
    \addplot[only marks, mark=x, color=black, line width = 0.5pt, mark size=2.000pt]
    coordinates {(0, 84.3197) (5, 185.9236)(10, 228.7985) (15, 237.3608) (20, 252.1172) (25, 255.1395) (30, 270.5735) (35, 270.4194) (40, 275.4801)};

     \addplot[name path=us_top,color=carminepink!70] coordinates {(0, 84.3197 + 2.1553) (5, 185.9236 + 4.4719) (10, 228.7985 + 6.4175) (15, 237.3608 + 3.7898)  (20, 252.1172 + 11.4732) (25, 255.1395 + 4.5254) (30, 270.5735 + 10.7893) (35, 270.4194 + 7.7727)  (40, 275.4801 + 3.6054) };

    \addplot[name path=us_down,color=carminepink!70] 
    coordinates {(0, 84.3197 - 2.1553) (5, 185.9236 - 4.4719) (10, 228.7985 - 6.4175) (15, 237.3608 - 3.7898)  (20, 252.1172 - 11.4732) (25, 255.1395 - 4.5254) (30, 270.5735 - 10.7893) (35, 270.4194 - 7.7727)  (40, 275.4801 - 3.6054) };

    \addplot[carminepink!50,fill opacity=0.5] fill between[of=us_top and us_down];
    \end{axis}
    
    \draw (2.6,.53) node {\footnotesize $\begin{aligned}K_{\tt T} & = 29.1067 \sqrt{\up} + 110.9847 \\ R^2_s& = 0.9216\end{aligned}$  };

    \draw (0.4, 1.95) node {\footnotesize (b)};
\end{tikzpicture}
\vspace{-.4cm}
 \\
        \\
    \end{tabular}
    \caption{Experiments for open-loop stiffening. (a) Experimental setup. (b) Experimental relation between $\up$ and the transverse stiffness under the jamming sheath with two layers. The experiments were repeated three times under each vacuum pressure. ``$\times$'' shows the mean values and color band represents the $\pm$1 standard deviation.}
    \label{fig:OpenStiff}
\end{figure}

\subsection{Parameter Identification}
To implement the proposed approach, we identified key plant parameters, which were used in the proposed controller, as discussed in Section \ref{sec:4}. The basic idea is that at any static configuration $q$ with $p =0$, the identity $\nabla_q U(q) = G(q) u$ holds true and it includes the key parameters to be identified. We regulated the continuum robot to different configurations, denoted as $\theta^j$ ($j=1,\ldots, w$ with some $w\in \mathbb{N}_+$) by driving the cables, and recorded the corresponding forces $(u_1^j, u_2^j)$. In the experiments, we selected $w = 15$ and repeated three times under each vacuum pressure, shown as Fig. \ref{fig:paraest}. Using this data set (in total 225 operation modes) and solving the optimization problem in \eqref{eq:opti}, we identified the key parameters as 
$
c_1= -0.0015, c_2 = 0.0890, c_3 = 1.7056,$ and $\alpha_1 = 1.0227.
$

\begin{figure}[htbp]
\centering
\begin{tikzpicture}
    \begin{axis}[
        width  = 0.85\linewidth,
        height = 0.5 \linewidth,
        xlabel={Bending angle $q_\star$ [deg]},
        ylabel={Mean force $u_\star$ [N]},
        ylabel style={yshift=-10pt},
        ymajorgrids=true,
        grid style=dashed,
        legend pos=north west,
        scale only axis,
        label style={font=\footnotesize},
        tick label style={font=\footnotesize},
        xtick={1,2,...,15},
        xmin=1,
        xmax=15,
        ymin=0,
        ymax=47,
        axis background/.style={fill=white},        
    ]
    \addplot[name path=upper0,draw=none] coordinates {
        (1, 3.0707 + 0.1612)
        (2, 4.4926 + 0.1957)
        (3, 5.9610 + 0.1281)
        (4, 7.5346 + 0.2305)
        (5, 9.1992 + 0.1081)
        (6, 10.6715 + 0.1363)
        (7, 11.7528 + 0.4159)
        (8, 12.7171 + 0.3406)
        (9, 13.9116 + 0.3638)
        (10, 15.9049 + 0.0629)
        (11, 17.3387 + 0.4745)
        (12, 18.3272 + 0.1798)
        (13, 19.8629 + 0.3898)
        (14, 22.1832 + 0.9300)
        (15, 24.9068 + 0.6315)
    };
    \addplot[name path=lower0,draw=none] coordinates {
        (1, 3.0707 - 0.1612)
        (2, 4.4926 - 0.1957)
        (3, 5.9610 - 0.1281)
        (4, 7.5346 - 0.2305)
        (5, 9.1992 - 0.1081)
        (6, 10.6715 - 0.1363)
        (7, 11.7528 - 0.4159)
        (8, 12.7171 - 0.3406)
        (9, 13.9116 - 0.3638)
        (10, 15.9049 - 0.0629)
        (11, 17.3387 - 0.4745)
        (12, 18.3272 - 0.1798)
        (13, 19.8629 - 0.3898)
        (14, 22.1832 - 0.9300)
        (15, 24.9068 - 0.6315)
    };
    \addplot[fill=black!20,opacity=0.6] fill between[of=upper0 and lower0];
    \addplot[only marks, mark=otimes, mark options={scale=0.6, black!70}] coordinates {
        (1, 3.0707)
        (2, 4.4926)
        (3, 5.9610)
        (4, 7.5346)
        (5, 9.1992)
        (6, 10.6715)
        (7, 11.7528)
        (8, 12.7171)
        (9, 13.9116)
        (10, 15.9049)
        (11, 17.3387)
        (12, 18.3272)
        (13, 19.8629)
        (14, 22.1832)
        (15, 24.9068)
    };

    \addplot[name path=upper10,draw=none] coordinates {
        (1, 4.0526 + 0.5596)
        (2, 6.6272 + 0.5089)
        (3, 9.0494 + 0.8559)
        (4, 11.9535 + 1.1069)
        (5, 14.8491 + 0.4792)
        (6, 16.6620 + 0.6812)
        (7, 19.1091 + 0.6451)
        (8, 21.6935 + 0.5955)
        (9, 24.2033 + 0.5958)
        (10, 26.1849 + 0.5335)
        (11, 28.9386 + 0.1691)
        (12, 31.3602 + 0.5644)
        (13, 33.8407 + 0.4407)
        (14, 36.7251 + 0.9261)
        (15, 40.1686 + 0.1125)        
    };
    \addplot[name path=lower10,draw=none] coordinates {
        (1, 4.0526 - 0.5596)
        (2, 6.6272 - 0.5089)
        (3, 9.0494 - 0.8559)
        (4, 11.9535 - 1.1069)
        (5, 14.8491 - 0.4792)
        (6, 16.6620 - 0.6812)
        (7, 19.1091 - 0.6451)
        (8, 21.6935 - 0.5955)
        (9, 24.2033 - 0.5958)
        (10, 26.1849 - 0.5335)
        (11, 28.9386 - 0.1691)
        (12, 31.3602 - 0.5644)
        (13, 33.8407 - 0.4407)
        (14, 36.7251 - 0.9261)
        (15, 40.1686 - 0.1125)        
    };
    \addplot[fill=cyan!30,opacity=0.6] fill between[of=upper10 and lower10];
    \addplot[only marks, mark=triangle, mark options={scale=0.6, cyan!80}] coordinates {
        (1, 4.0526)
        (2, 6.6272)
        (3, 9.0494)
        (4, 11.9535)
        (5, 14.8491)
        (6, 16.6620)
        (7, 19.1091)
        (8, 21.6935)
        (9, 24.2033)
        (10, 26.1849)
        (11, 28.9386)
        (12, 31.3602)
        (13, 33.8407)
        (14, 36.7251)
        (15, 40.1686)
    };

    \addplot[name path=upper20,draw=none] coordinates {
        (1, 3.3975 + 0.5585)
        (2, 6.0695 + 0.8281)
        (3, 8.9984 + 0.9641)
        (4, 11.9542 + 0.9596)
        (5, 14.8779 + 0.5101)
        (6, 17.4197 + 0.6587)
        (7, 20.1355 + 0.8960)
        (8, 22.7650 + 0.6156)
        (9, 25.4854 + 0.8319)
        (10, 27.8900 + 0.3657)
        (11, 30.6012 + 0.2257)
        (12, 33.6674 + 0.5864)
        (13, 36.8958 + 0.1678)
        (14, 39.7214 + 0.9573)
        (15, 42.9576 + 1.1358)        
    };

    \addplot[name path=lower20,draw=none] coordinates {
        (1, 3.3975 - 0.5585)
        (2, 6.0695 - 0.8281)
        (3, 8.9984 - 0.9641)
        (4, 11.9542 - 0.9596)
        (5, 14.8779 - 0.5101)
        (6, 17.4197 - 0.6587)
        (7, 20.1355 - 0.8960)
        (8, 22.7650 - 0.6156)
        (9, 25.4854 - 0.8319)
        (10, 27.8900 - 0.3657)
        (11, 30.6012 - 0.2257)
        (12, 33.6674 - 0.5864)
        (13, 36.8958 - 0.1678)
        (14, 39.7214 - 0.9573)
        (15, 42.9576 - 1.1358)
    };
    \addplot[fill=blue!30, opacity=0.6] fill between[of=upper20 and lower20];
    \addplot[only marks, mark=x, mark options={scale=0.6, blue!80}] coordinates {
        (1, 3.3975)
        (2, 6.0695)
        (3, 8.9984)
        (4, 11.9542)
        (5, 14.8779)
        (6, 17.4197)
        (7, 20.1355)
        (8, 22.7650)
        (9, 25.4854)
        (10, 27.8900)
        (11, 30.6012)
        (12, 33.6674)
        (13, 36.8958)
        (14, 39.7214)
        (15, 42.9576)        
    };

    \addplot[name path=upper30,draw=none] coordinates {
        (1, 2.5574 + 0.0708)
        (2, 5.1104 + 0.3881)
        (3, 7.9871 + 0.4005)
        (4, 10.4891 + 0.7747)
        (5, 14.5327 + 0.4862)
        (6, 17.5407 + 0.3957)
        (7, 20.2937 + 0.5401)
        (8, 23.0645 + 0.3088)
        (9, 26.0980 + 0.7738)
        (10, 28.5882 + 0.6134)
        (11, 31.5204 + 0.3649)
        (12, 34.8475 + 0.1453)
        (13, 37.7659 + 0.5357)
        (14, 40.8910 + 0.8441)
        (15, 43.9703 + 0.6496)
    };
    \addplot[name path=lower30,draw=none] coordinates {
        (1, 2.5574 - 0.0708)
        (2, 5.1104 - 0.3881)
        (3, 7.9871 - 0.4005)
        (4, 10.4891 - 0.7747)
        (5, 14.5327 - 0.4862)
        (6, 17.5407 - 0.3957)
        (7, 20.2937 - 0.5401)
        (8, 23.0645 - 0.3088)
        (9, 26.0980 - 0.7738)
        (10, 28.5882 - 0.6134)
        (11, 31.5204 - 0.3649)
        (12, 34.8475 - 0.1453)
        (13, 37.7659 - 0.5357)
        (14, 40.8910 - 0.8441)
        (15, 43.9703 - 0.6496)
    };
    \addplot[fill=green!20,opacity=0.5] fill between[of=upper30 and lower30];
    \addplot[only marks, mark=diamond, mark options={scale=0.6, green!80}] coordinates{
        (1, 2.5574)
        (2, 5.1104)
        (3, 7.9871)
        (4, 10.4891)
        (5, 14.5327)
        (6, 17.5407)
        (7, 20.2937)
        (8, 23.0645)
        (9, 26.0980)
        (10, 28.5882)
        (11, 31.5204)
        (12, 34.8475)
        (13, 37.7659)
        (14, 40.8910)
        (15, 43.9703)
    };

    \addplot[name path=upper40,draw=none] coordinates {
        (1, 2.8634 + 0.5157)
        (2, 5.5504 + 0.8293)
        (3, 8.4643 + 0.8836)
        (4, 11.1919 + 1.3728)
        (5, 15.1773 + 0.6366)
        (6, 18.1572 + 0.5928)
        (7, 21.0162 + 0.6649)
        (8, 24.1386 + 0.8571)
        (9, 26.6903 + 0.2459)
        (10, 29.4943 + 0.1743)
        (11, 32.3010 + 0.0432)
        (12, 35.3783 + 0.0430)
        (13, 38.6525 + 0.0460)
        (14, 41.7226 + 0.2720)
        (15, 45.3216 + 0.1475)
    };
    \addplot[name path=lower40,draw=none] coordinates {
        (1, 2.8634 - 0.5157)
        (2, 5.5504 - 0.8293)
        (3, 8.4643 - 0.8836)
        (4, 11.1919 - 1.3728)
        (5, 15.1773 - 0.6366)
        (6, 18.1572 - 0.5928)
        (7, 21.0162 - 0.6649)
        (8, 24.1386 - 0.8571)
        (9, 26.6903 - 0.2459)
        (10, 29.4943 - 0.1743)
        (11, 32.3010 - 0.0432)
        (12, 35.3783 - 0.0430)
        (13, 38.6525 - 0.0460)
        (14, 41.7226 - 0.2720)
        (15, 45.3216 - 0.1475)
    };
    \addplot[fill=red!30,opacity=0.6] fill between[of=upper40 and lower40];
    \addplot[only marks, mark=o, mark options={scale=0.6, red!90}] coordinates {
        (1, 2.8634)
        (2, 5.5504)
        (3, 8.4643)
        (4, 11.1919)
        (5, 15.1773)
        (6, 18.1572)
        (7, 21.0162)
        (8, 24.1386)
        (9, 26.6903)
        (10, 29.4943)
        (11, 32.3010)
        (12, 35.3783)
        (13, 38.6525)
        (14, 41.7226)
        (15, 45.3216)
    };
    \node[anchor=north west, color=black] at (rel axis cs:0.03,0.82) {$\otimes$ 0 kPa};
    \node[anchor=north west, color=cyan] at (rel axis cs:0.23,0.82) {$\triangle$ 10 kPa};
    \node[anchor=north west, color=blue] at (rel axis cs:0.45,0.82) {$\times$ 20 kPa};
    \node[anchor=north west, color=green] at (rel axis cs:0.15,0.73) {$\diamond$ 30 kPa};
    \node[anchor=north west, color=red] at (rel axis cs:0.36,0.73) {$\circ$ 40 kPa};
    \end{axis}
\end{tikzpicture}
\vspace{-5mm}
\caption{Data set for parameter identification. All the marks show mean values and error bands represent the $\pm$1 standard deviation.}
\label{fig:paraest}
\end{figure}

Note that Fig. \ref{fig:OpenStiff}(b) only shows the open-loop stiffness at the zero configuration $\theta_\star=0$; see \eqref{eq:K}. In order to control the closed-loop stiffness, it is also necessary to estimate the function $\phi(\cdot)$. To this end, we need to identify the function $\alpha_2(\up)$. We noted that at $\theta_\star = 0$, the terms $\alpha_1$ and $\alpha_2(\cdot)$ are relatively negligible compared to the last term in \eqref{open_loop_eq} for their contribution to stiffness. Therefore, we treated the 
$
K|_{q_\star = 0} \approx \sigma_0 \phi(\up^\star)I_3.
$\footnote{Such an approximation simplifies the parameter identification process and remain valid, as the contribution of the term $\alpha_2(\up)$ to the overall stiffness is relatively minor. Detailed results are provided in Table~\ref{tab:3}.}
By fitting the stiffness data shown in Fig. \ref{fig:OpenStiff}(b) in the above equation, we obtain the estimates $
c_4= 110.9847,$ and  $c_5 = 29.1067.
$
Note that the parameters $c_4$ and $c_5$ include the bristle stiffness coefficient $\sigma_0$.

\subsection{Closed-loop Control Experiments}

To evaluate the performance of the proposed configuration and stiffness controller, we considered the desired homogeneous configurations ${q}_\star = [\theta_\star, \ldots, \theta_\star]^\top$, with three different values, i.e. $\theta_\star = $ 5, 10 and 15 deg\footnote{Note that $\theta_\star $ is the segment bending angle and there are six segments in our robot platform. Therefore, the real robot bending angles were $6\theta_\star$, and the maximum angle is limited to 60 deg due to its structural design.}. We conducted experiments for the cases without external perturbation under five vacuum pressures $\up =$ 0, 10, 20, 30, and 40~kPa with the gains $\gamma = 0.1$ and $K_{\tt D} = 1$. Note that the gains $\gamma$ and $K_{\tt D}$ play similar roles to those in recent work \cite{YIetal23}. Interested readers may refer to it for further details on their effects and tuning methods. 

The steady-state accuracy for these experiments under different equilibria and vacuum pressures is summarized in Table~\ref{tab:2}. Here, $\cali_{\tt B} \subset \rea_+$ is the time interval of steady-state. For example, $\cali_B=[7,10]$~s means that we consider the steady-state performance within 7-10 seconds. $[\theta_{\tt min}, \theta_{\tt max}]$ represents the minimum and maximum values of the configuration variable during the interval $\cali_{\tt B}$. The table also provides the root mean square (RMS) and mean absolute error (MAE) for each scenario.

\begin{table}[h]
\caption{Steady-state errors over the time interval $\cali_B$ for different scenarios (Unit: deg)}
\label{tab:2}\footnotesize

\setlength{\tabcolsep}{7pt} 
\renewcommand{\arraystretch}{1.5}  

\centering
\begin{tabular}{c c c c c c}
\specialrule{1pt}{0pt}{0pt} 
$\up$ & $\theta_\star$ & $\cali_B$ (s) &{[}$\theta_{\tt min}, \theta_{\tt max}${]} & RMS  & MAE \\ \cline{1-6}

\multirow{3}{*}{0} & 5$^{*}$ & [7, 10] & [4.8120, 4.9596] & 4.8176   &    0.1824    \\  
& 10 & [2, 5] & [9.7420, 9.8600] & 9.7784   &   0.2216     \\  
& 15 & [2, 5] &[14.8056, 14.9997] & 14.9821   &    0.0179  \\ 

\multirow{3}{*}{10} & 5 & [6, 9] & [4.4162, 4.9989] & 4.9855   &    0.0152    \\  
& 10$^{*}$ & [5, 8] & [10.0690, 10.1251] & 10.1100   &   0.1100     \\  
& 15 & [14, 17] &[14.3030, 14.2173] & 14.2183   &    0.7826  \\ 

\multirow{3}{*}{20} & 5$^{*}$ & [13, 16] & [4.6976, 4.9998] & 4.9823   &    0.0182    \\  
& 10 & [11, 14] & [10.2318, 10.2952] & 10.2759   &   0.2758     \\  
& 15 & [14, 17] &[14.3234, 14.5760] & 14.5303   &    0.4697  \\

\multirow{3}{*}{30} & 5 & [8, 11] & [4.3131, 4.9998] & 4.9706   &    0.0311    \\  
& 10$^{*}$ & [8, 11] & [10.1430, 10.2844] & 10.1952   &   0.1951     \\  
& 15 & [11, 14] &[14.4322, 14.6202] & 14.4641   &    0.5360  \\ 

\multirow{3}{*}{40} & 5 & [7, 10] & [4.7629, 4.9987] & 4.9939   &    0.0385    \\  
& 10 & [7, 10] & [9.6650, 9.7374] & 9.6786   &   0.3214     \\  
& 15$^{*}$ & [13, 16] &[14.4489, 14.9997] & 14.9932   &    0.5783  \\ 

\specialrule{1pt}{0pt}{0pt}
\end{tabular}
\end{table}

The results shown in Table~\ref{tab:2} indicate that the proposed controller is capable of achieving high accuracy and performance for configuration control. For the desired configurations $\theta_\star = 5$ and 10 deg at varying vacuum pressures, the highest accuracy was achieved when the vacuum pressure $\up$ = 10~kPa. For the desired configuration $\theta_\star = 15$ deg, the best performance was attained when no vacuum pressure was applied. Comparing the three desired configurations, regulating to $\theta_\star = 5$ deg consistently yields better performance when vacuum pressures are maintained. In all scenarios, the transient stages lasted less than 15 seconds, during which the configuration state $q$ rapidly converged to small neighborhoods around the desired configurations. This highlights the high accuracy of the proposed controller and the efficiency of the proposed model for LJ-based continuum robots.

\begin{figure*}[!htp]
\centering
\includegraphics[width = 1\linewidth]{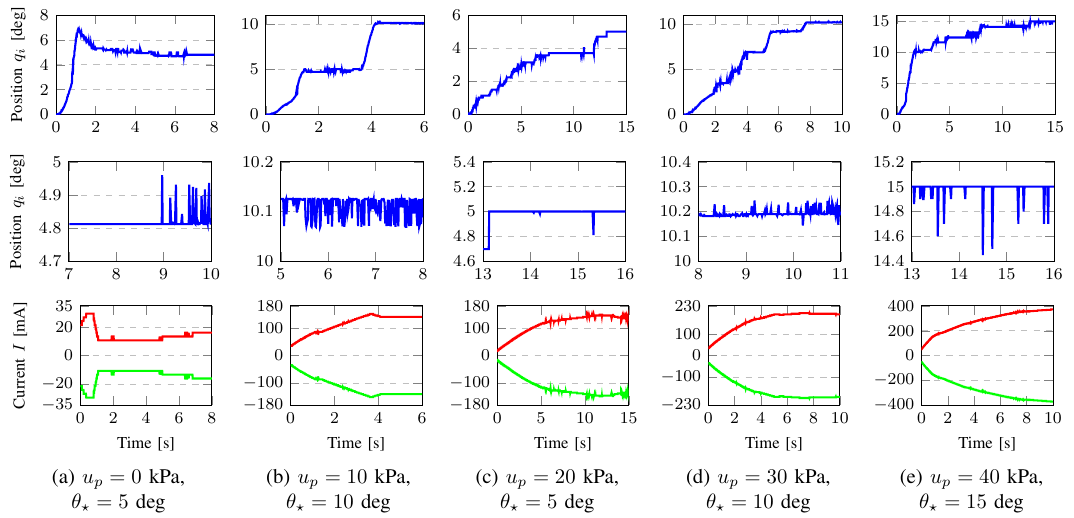} 
\caption{Configuration control performance of the proposed feedback controller.}
\label{fig:ctrl}
\end{figure*}

To further visualize the control performance, we present in Fig.~\ref{fig:ctrl} the control experimental results of the example scenarios---marked with asterisk ($*$) in Table~\ref{tab:2}. The first row of the figure shows the configurations over longer intervals, while the second row shows the trajectories of the configuration variables during the steady-state stage within the time interval $\cali_B$. The current input of the motors is given in the third row of the figures. Note that negative current values indicate that the motor is retracting the cable, while positive values signify the cable is being released.

To verify that the proposed design can concurrently regulate closed-loop stiffness, we equipped a linear actuator perpendicularly to the tangential direction of the continuum robot at the end-effector, as shown in Fig.~\ref{fig:TestRig}(a). Stiffness data were collected using varying control gains $\gamma$, and the results are plotted in Fig. \ref{fig:Stiff}. We then quantified the contributions of the terms $\gamma$ and $\alpha_2(\up)$ to closed-loop stiffness and expressed their percentages, as shown in Table~\ref{tab:3}.

\begin{figure}[h]
    \centering    
    \begin{tikzpicture}
\begin{axis}[%
width  = 0.8\linewidth,
height = 0.45 \linewidth,
ymajorgrids=true,
grid style=dashed,
legend pos=north west,
scale only axis,
label style={font=\footnotesize},
tick label style={font=\footnotesize},
xmin=0,
xmax=8,
ymin=0,
ymax=320,
ylabel={Stiffness [N/m]},
axis background/.style={fill=white},
legend style={at={(0.5,0.6)},draw=none,anchor=north}, 
legend columns=3, 
]

\addplot[mark= pentagon, mark options={solid}, mark size=2pt, dashed,color=orange] table[row sep=crcr]{%
x	y\\
0.1 77\\
2 82.4\\
4 81.6\\
6 87.6\\
8 91.6\\
};
\addlegendentry{0 kPa} 

\addplot[mark= diamond, mark options={solid}, mark size=2pt, dashed, color=green] table[row sep=crcr]{%
x	y\\
0.1 235.1 \\
2 257.7 \\
4 254.4 \\
6 262.7 \\
8 270.9 \\
};
\addlegendentry{10 kPa}

\addplot[mark= triangle, mark options={solid}, mark size=2pt, color=blue, dashed] table[row sep=crcr]{%
x	y\\
0.1 245.2\\
2 258.9\\
4 265.7\\
6 274.8\\
8 286.9\\
};
\addlegendentry{20 kPa}

\addplot[mark= otimes, mark options={solid}, mark size=2pt, dashed, color=cyan] table[row sep=crcr]{%
x	y\\
0.1 254.7\\
2 268.0\\
4 275.8\\
6 276.7\\
8 287.7\\
};
\addlegendentry{30 kPa}

\addplot[mark= square, mark options={solid}, mark size=2pt, color=red, dashed] table[row sep=crcr]{%
x	y\\
0.1 261.4\\
2 278.6\\
4 283.2\\
6 285.1\\
8 299.3\\
};
\addlegendentry{40 kPa}
\node[anchor=north west] at (rel axis cs:0,1) {(a)};
\end{axis}
\end{tikzpicture}%
\\
    \begin{tikzpicture}
\begin{axis}[%
width  = 0.8\linewidth,
height = 0.45 \linewidth,
ymajorgrids=true,
grid style=dashed,
legend pos=north west,
scale only axis,
label style={font=\footnotesize},
tick label style={font=\footnotesize},
xmin=0,
xmax=8,
xlabel={The gain $\gamma$},
ymin=0,
ymax=300,
ylabel={Stiffness [N/m]},
axis background/.style={fill=white},
legend style={at={(0.5,0.65)},draw=none,anchor=north}, 
legend columns=3, 
]

\addplot[mark= pentagon, mark options={solid}, mark size=2pt, dashed,color=orange] table[row sep=crcr]{%
x	y\\
0.1 67.9\\
2 86.8\\
4 88.3\\
6 97.7\\
8 106.3\\
};
\addlegendentry{0 kPa} 

\addplot[mark= diamond, mark options={solid}, mark size=2pt, dashed, color=green] table[row sep=crcr]{%
x	y\\
0.1 193.3\\
2 203.5\\
4 208.5\\
6 210.3\\
8 219.1\\
};
\addlegendentry{10 kPa}

\addplot[mark= triangle, mark options={solid}, mark size=2pt, color=blue, dashed] table[row sep=crcr]{%
x	y\\
0.1 223.6\\
2 238.1\\
4 235.6\\
6 233.4\\
8 255.6\\
};
\addlegendentry{20 kPa}

\addplot[mark= otimes, mark options={solid}, mark size=2pt, dashed, color=cyan] table[row sep=crcr]{%
x	y\\
0.1 247.3\\
2 255.2\\
4 257.8\\
6 264.4\\
8 274.2\\
};
\addlegendentry{30 kPa}

\addplot[mark= square, mark options={solid}, mark size=2pt, color=red, dashed] table[row sep=crcr]{%
x	y\\
0.1 257.1\\
2 267.0\\
4 273.2\\
6 274.9\\
8 282.7\\
};
\addlegendentry{40 kPa}
\node[anchor=north west] at (rel axis cs:0,1) {(b)};
\end{axis}
\end{tikzpicture}%
    \caption{Stiffness regulation with different vacuum pressures $\up$ and the gains $\gamma$. (a) Desired configuration $\theta_\star =$ 8 deg. (b) Desired configuration $\theta_\star =$ 10 deg.}
    \label{fig:Stiff}
\end{figure}

The results were obtained by normalizing the stiffness values from Fig.~\ref{fig:Stiff} with respect to $\gamma$ and $\alpha_2$, 
respectively. In Table~\ref{tab:3}, the results are presented in the format of $[r_1,r_2]$, where $r_1$ and $r_2$ correspond to the contributions from $\gamma$ and $\alpha_2$, respectively. From Fig.~\ref{fig:Stiff} and Table~\ref{tab:3}, we note that both $\gamma$ and $\alpha_2$ have a minor influence on the closed-loop stiffness, with maximum contributions of 8.73~\% for $\gamma$ and 3.55~\% for $\alpha_2$. As the vacuum pressure $\up$ increases, the relative contributions of $\gamma$ and $\alpha_2$ to the overall stiffness decrease.

The observed outcomes are consistent with the equation \eqref{eq:K2}, which shows that closed-loop stiffness is primarily affine in the vacuum pressure $\up$ (mainly contributed by the term of $\phi(\up)$), with only a slight influence from the control gain $\gamma$. This indicates that stiffness and configuration can be almost decoupled, allowing stiffness control to be achieved primarily by adjusting the vacuum pressure  $\up$.

\begin{table}[h]
\caption{Stiffness ratios of $\gamma$ and $\alpha_2$ (Unit: \%, represented as $[r_1,r_2]$, where $r_1$ and $r_2$ correspond to the contributions from $\gamma$ and $\alpha_2$, respectively)}
\label{tab:3}\footnotesize

\setlength{\tabcolsep}{3pt} 
\renewcommand{\arraystretch}{1.5}  

\centering
\begin{tabular}{c c c c c c c}
\specialrule{1pt}{0pt}{0pt} 
$\up$ & $\theta_\star$ & $\gamma$ = 0.1 & $\gamma$ = 2 & $\gamma$ = 4 & $\gamma$ = 6 &$\gamma$ = 8\\ \cline{1-7}

\multirow{2}{*}{0} & 8 &
        [0.13, 2.22] &
        [2.43, 2.97] &
        [4.90, 3.55] &
        [6.85, 3.47] &
        [8.73, 3.16] \\  
        & 10 &
        [0.15, 2.51] &
     [2.30, 2.82] &
     [4.53, 3.28] &
     [6.14, 3.11] &
     [7.53, 2.72] \\

\multirow{2}{*}{10} & 8 &
[0.04, 0.73] &
[0.78, 0.95] &
[1.57, 1.14] &
[2.28, 1.16] &
[2.95, 1.07]\\
        & 10 & [0.05, 0.88] &
     [0.98, 1.20] &
     [1.92, 1.39] &
     [2.85, 1.45] &
     [3.65, 1.32] \\

\multirow{2}{*}{20} & 8 &
[0.04, 0.70] &
[0.77, 0.95] &
[1.51, 1.09] &
[2.18, 1.11] &
[2.79, 1.01] \\ 
        & 10 & [0.04, 0.76] &
     [0.84, 1.03] &
     [1.70, 1.23] &
     [2.57, 1.30] &
     [3.13, 1.13] \\

\multirow{2}{*}{30} & 8 &
[0.04, 0.67] &
[0.75, 0.91] &
[1.45, 1.05] &
[2.17, 1.10] &
[2.78, 1.01] \\
        & 10 & [0.04, 0.69] &
     [0.78, 0.96] &
     [1.55, 1.12] &
     [2.27, 1.15] &
     [2.92, 1.06] \\

\multirow{2}{*}{40} & 8 &
[0.04, 0.65] &
[0.72, 0.88] &
[1.41, 1.02] &
[2.10, 1.07] &
[2.67, 0.97] \\
        & 10 & [0.04, 0.66] &
     [0.75, 0.92] &
     [1.46, 1.06] &
     [2.18, 1.11] &
     [2.83, 1.02] \\        
        
\specialrule{1pt}{0pt}{0pt}
\end{tabular}
\end{table}

\section{Concluding Remarks}
\label{sec:6}
In this paper, we propose a novel dynamical model for layer jamming-based continuum robots, which integrates an energy-based modeling approach with the LuGre frictional model. Using the proposed model, we theoretically analyze its dynamical behavior and show its effectiveness in interpreting two key phenomena---shape locking and adjustable stiffness---in this type of robot with quantitative results. As mentioned at the beginning, the motivation of this work is to propose a \emph{control-oriented} model. Accordingly, we further study feedback controller design based on this model to simultaneously regulate both the robot’s configuration and closed-loop stiffness while ensuring some guaranteed stability and convergence properties. These results have been experimentally verified on our robotic platform, showing the efficiency of both the proposed model and the designed controller. Compared to recent work \cite{YIetal23}, the new design allows for more flexible stiffness adjustment over a larger range for the class of LJ-based continuum robots. As future work, it would be of practical interest to extend the proposed model by integrating direct task-to-actuator inversion \cite{DEL2025}, enabling more effective handling of complex control challenges in underactuated continuum robots. While this study focuses on a particular case: tendon-driven actuation with a rigid-link approximation, the underlying idea is promising for extension to other actuation mechanisms and the Cosserat model.




\bibliography{reference.bib}

\begin{thebibliography}{10}
\providecommand{\url}[1]{#1}
\csname url@samestyle\endcsname
\providecommand{\newblock}{\relax}
\providecommand{\bibinfo}[2]{#2}
\providecommand{\BIBentrySTDinterwordspacing}{\spaceskip=0pt\relax}
\providecommand{\BIBentryALTinterwordstretchfactor}{4}
\providecommand{\BIBentryALTinterwordspacing}{\spaceskip=\fontdimen2\font plus
\BIBentryALTinterwordstretchfactor\fontdimen3\font minus \fontdimen4\font\relax}
\providecommand{\BIBforeignlanguage}[2]{{%
\expandafter\ifx\csname l@#1\endcsname\relax
\typeout{** WARNING: IEEEtran.bst: No hyphenation pattern has been}%
\typeout{** loaded for the language `#1'. Using the pattern for}%
\typeout{** the default language instead.}%
\else
\language=\csname l@#1\endcsname
\fi
#2}}
\providecommand{\BIBdecl}{\relax}
\BIBdecl

\bibitem{Burgner}
J.~Burgner-Kahrs, D.~C. Rucker, and H.~Choset, ``Continuum robots for medical applications: A survey,'' \emph{IEEE Transactions on Robotics}, vol.~31, no.~6, pp. 1261--1280, 2015.

\bibitem{BAJSIM}
A.~Bajo and N.~Simaan, ``Hybrid motion/force control of multi-backbone continuum robots,'' \emph{The International Journal of Robotics Research}, vol.~35, no.~4, pp. 422--434, 2016.

\bibitem{NARetalRAL}
Y.~S. Narang, A.~Degirmenci, J.~J. Vlassak, and R.~D. Howe, ``Transforming the dynamic response of robotic structures and systems through laminar jamming,'' \emph{IEEE Robotics and Automation Letters}, vol.~3, no.~2, pp. 688--695, 2017.

\bibitem{CLAROJ}
A.~B. Clark and N.~Rojas, ``Assessing the performance of variable stiffness continuum structures of large diameter,'' \emph{IEEE Robotics and Automation Letters}, vol.~4, no.~3, pp. 2455--2462, 2019.

\bibitem{YANetal}
C.~Yang, S.~Geng, I.~Walker, D.~T. Branson, J.~Liu, J.~S. Dai, and R.~Kang, ``Geometric constraint-based modeling and analysis of a novel continuum robot with shape memory alloy initiated variable stiffness,'' \emph{The International Journal of Robotics Research}, vol.~39, no.~14, pp. 1620--1634, 2020.

\bibitem{CLAROJtro}
A.~B. Clark and N.~Rojas, ``Malleable robots: Reconfigurable robotic arms with continuum links of variable stiffness,'' \emph{IEEE Transactions on Robotics}, vol.~38, no.~6, pp. 3832--3849, 2022.

\bibitem{SANetal}
J.~L.~C. Santiago, I.~S. Godage, P.~Gonthina, and I.~D. Walker, ``Soft robots and kangaroo tails: Modulating compliance in continuum structures through mechanical layer jamming,'' \emph{Soft Robotics}, vol.~3, no.~2, pp. 54--63, 2016.

\bibitem{LANetal}
M.~Langer, E.~Amanov, and J.~Burgner-Kahrs, ``Stiffening sheaths for continuum robots,'' \emph{Soft Robotics}, vol.~5, no.~3, pp. 291--303, 2018.

\bibitem{FAN24RCIM}
Y.~Fan, B.~Yi, and D.~Liu, ``An overview of stiffening approaches for continuum robots,'' \emph{Robotics and Computer-Integrated Manufacturing}, vol.~90, p. 102811, 2024.

\bibitem{KIMetalIROS}
Y.-J. Kim, S.~Cheng, S.~Kim, and K.~Iagnemma, ``Design of a tubular snake-like manipulator with stiffening capability by layer jamming,'' in \emph{IEEE/RSJ International Conference on Intelligent Robots and Systems (IROS)}.\hskip 1em plus 0.5em minus 0.4em\relax IEEE, 2012, pp. 4251--4256.

\bibitem{KIMetal}
------, ``A novel layer jamming mechanism with tunable stiffness capability for minimally invasive surgery,'' \emph{IEEE Transactions on Robotics}, vol.~29, no.~4, pp. 1031--1042, 2013.

\bibitem{SCISIC}
L.~Sciavicco and B.~Siciliano, \emph{Modelling and Control of Robot Manipulators}.\hskip 1em plus 0.5em minus 0.4em\relax Springer Science \& Business Media, 2012.

\bibitem{DELetalREV}
C.~Della~Santina, C.~Duriez, and D.~Rus, ``Model-based control of soft robots: A survey of the state of the art and open challenges,'' \emph{IEEE Control Systems Magazine}, vol.~43, no.~3, pp. 30--65, 2023.

\bibitem{DELetal}
C.~Della~Santina, R.~K. Katzschmann, A.~Bicchi, and D.~Rus, ``Model-based dynamic feedback control of a planar soft robot: Trajectory tracking and interaction with the environment,'' \emph{The International Journal of Robotics Research}, vol.~39, no.~4, pp. 490--513, 2020.

\bibitem{FRAGAR}
E.~Franco and A.~Garriga-Casanovas, ``Energy-shaping control of soft continuum manipulators with in-plane disturbances,'' \emph{The International Journal of Robotics Research}, vol.~40, no.~1, pp. 236--255, 2021.

\bibitem{CAAetal}
B.~Caasenbrood, A.~Pogromsky, and H.~Nijmeijer, ``Energy-shaping controllers for soft robot manipulators through port-hamiltonian {C}osserat models,'' \emph{SN Computer Science}, vol.~3, no.~6, p. 494, 2022.

\bibitem{NARetal}
Y.~S. Narang, J.~J. Vlassak, and R.~D. Howe, ``Mechanically versatile soft machines through laminar jamming,'' \emph{Advanced Functional Materials}, vol.~28, no.~17, pp. 1--9, 2018, art. no. 1707136.

\bibitem{ZHAetalBB}
Y.~Zhao, Y.~Shan, J.~Zhang, K.~Guo, L.~Qi, L.~Han, and H.~Yu, ``A soft continuum robot, with a large variable-stiffness range, based on jamming,'' \emph{Bioinspiration \& Biomimetics}, vol.~14, no.~6, 2019, art. no. 066007.

\bibitem{CHENetal}
C.~Chen, H.~Ren, and H.~Wang, ``Augment laminar jamming variable stiffness through electroadhesion and vacuum actuation,'' \emph{IEEE Transactions on Robotics}, vol.~41, pp. 819--836, 2025.

\bibitem{CARetal}
F.~Caruso, G.~Mantriota, L.~Afferrante, and G.~Reina, ``A theoretical model for multi-layer jamming systems,'' \emph{Mechanism and Machine Theory}, vol. 172, 2022, art. no. 104788.

\bibitem{DOetal}
B.~H. Do, I.~Choi, and S.~Follmer, ``An all-soft variable impedance actuator enabled by embedded layer jamming,'' \emph{IEEE/ASME Transactions on Mechatronics}, vol.~27, no.~6, pp. 5529--5540, 2022.

\bibitem{IBRetal}
M.~Ibrahimi, L.~Patern{\`o}, L.~Ricotti, and A.~Menciassi, ``A layer jamming actuator for tunable stiffness and shape-changing devices,'' \emph{Soft Robotics}, vol.~8, no.~1, pp. 85--96, 2021.

\bibitem{YIetal23}
B.~Yi, Y.~Fan, D.~Liu, and J.~Guadalupe~Romero, ``Simultaneous position-and-stiffness control of underactuated antagonistic tendon-driven continuum robots,'' \emph{IEEE Transactions on Automation Science and Engineering}, vol.~22, pp. 7238--7254, 2025.

\bibitem{VAN}
A.~van~der Schaft, \emph{$L_2$-Gain and Passivity Techniques in Nonlinear Control}, 3rd~ed.\hskip 1em plus 0.5em minus 0.4em\relax Springer, 2017.

\bibitem{ORTetalBOOK}
R.~Ortega, A.~Loria, P.~J. Nicklasson, and H.~Sira-Ramirez, \emph{Passivity-based Control of Euler-Lagrange Systems: Mechanical, Electrical and Electromechanical Applications}.\hskip 1em plus 0.5em minus 0.4em\relax Springer, 1998.

\bibitem{FERetal}
J.~Ferguson, A.~Donaire, and R.~H. Middleton, ``Integral control of port-hamiltonian systems: Nonpassive outputs without coordinate transformation,'' \emph{IEEE Transactions on Automatic Control}, vol.~62, no.~11, pp. 5947--5953, 2017.

\bibitem{YIetalAUT}
B.~Yi, R.~Ortega, D.~Wu, and W.~Zhang, ``Orbital stabilization of nonlinear systems via {Mexican} sombrero energy shaping and pumping-and-damping injection,'' \emph{Automatica}, vol. 112, 2020, art. no. 108661.

\bibitem{BORetal}
P.~Borja, A.~Dabiri, and C.~Della~Santina, ``Energy-based shape regulation of soft robots with unactuated dynamics dominated by elasticity,'' in \emph{IEEE International Conference on Soft Robotics (RoboSoft)}.\hskip 1em plus 0.5em minus 0.4em\relax IEEE, 2022, pp. 396--402.

\bibitem{PAGetal}
G.~Pagnanelli, M.~Pierallini, F.~Angelini, and A.~Bicchi, ``Assessing an energy-based control for the soft inverted pendulum in {Hamiltonian} form,'' \emph{IEEE Control Systems Letters}, vol.~8, pp. 922--927, 2024.

\bibitem{HADSHA}
S.~Haddadin and E.~Shahriari, ``Unified force-impedance control,'' \emph{The International Journal of Robotics Research}, vol.~43, no.~13, pp. 2112--2141, 2024.

\bibitem{ASTDEW}
K.~J. Aström and C.~Canudas-De-Wit, ``Revisiting the {LuGre} friction model,'' \emph{IEEE Control Systems Magazine}, vol.~28, no.~6, pp. 101--114, 2008.

\bibitem{YIetalICRA}
B.~Yi, Y.~Fan, and D.~Liu, ``A novel model for layer jamming-based continuum robots,'' in \emph{IEEE International Conference on Robotics \& Automation (ICRA)}, 2024, pp. 12\,727--12\,733.

\bibitem{CHOetal}
W.~H. Choi, S.~Kim, D.~Lee, and D.~Shin, ``Soft, multi-{DoF}, variable stiffness mechanism using layer jamming for wearable robots,'' \emph{IEEE Robotics and Automation Letters}, vol.~4, no.~3, pp. 2539--2546, 2019.

\bibitem{JADetal}
S.~Jadhav, M.~R.~A. Majit, B.~Shih, J.~P. Schulze, and M.~T. Tolley, ``Variable stiffness devices using fiber jamming for application in soft robotics and wearable haptics,'' \emph{Soft Robotics}, vol.~9, no.~1, pp. 173--186, 2022.

\bibitem{KIMPARetal}
N.~Kim, J.~Park, and D.~Shin, ``Impedance for assistance: Upper-limb assistive soft robotic suit using linked-layer jamming mechanisms,'' \emph{Soft Robotitcs}, vol.~11, no.~6, pp. 970--983, 2024.

\bibitem{FANetal}
Y.~Fan, D.~Liu, and L.~Ye, ``A novel continuum robot with stiffness variation capability using layer jamming: Design, modeling, and validation,'' \emph{IEEE Access}, vol.~10, pp. 130\,253--130\,263, 2022.

\bibitem{DEWetal}
C.~C. De~Wit, H.~Olsson, K.~J. Astrom, and P.~Lischinsky, ``A new model for control of systems with friction,'' \emph{IEEE Transactions on Automatic Control}, vol.~40, no.~3, pp. 419--425, 1995.

\bibitem{KOOetal}
J.~Koopman, D.~Jeltsema, and M.~Verhaegen, ``Port-{Hamiltonian} description and analysis of the {LuGre} friction model,'' \emph{Simulation Modelling Practice and Theory}, vol.~19, no.~3, pp. 959--968, 2011.

\bibitem{BARORT}
N.~Barahanov and R.~Ortega, ``Necessary and sufficient conditions for passivity of the {LuGre} friction model,'' \emph{IEEE Transactions on Automatic Control}, vol.~45, no.~4, pp. 830--832, 2000.

\bibitem{KHAbook}
H.~K. Khalil, \emph{Nonlinear Systems}, 3rd~ed.\hskip 1em plus 0.5em minus 0.4em\relax Prentice Hall, 2001.

\bibitem{ORTetal01}
R.~Ortega, A.~J. Van Der~Schaft, I.~Mareels, and B.~Maschke, ``Putting energy back in control,'' \emph{IEEE Control Systems Magazine}, vol.~21, no.~2, pp. 18--33, 2001.

\bibitem{YIMAN}
B.~Yi and I.~R. Manchester, ``On the equivalence of contraction and {Koopman} approaches for nonlinear stability and control,'' \emph{IEEE Transactions on Automatic Control}, vol.~69, no.~7, pp. 4336--4351, 2023.

\bibitem{PILetal}
G.~Pillonetto, F.~Dinuzzo, T.~Chen, G.~De~Nicolao, and L.~Ljung, ``Kernel methods in system identification, machine learning and function estimation: A survey,'' \emph{Automatica}, vol.~50, no.~3, pp. 657--682, 2014.

\bibitem{WILetal}
C.~Williams and C.~Rasmussen, ``Gaussian processes for regression,'' \emph{Advances in Neural Information Processing Systems}, vol.~8, 1995.

\bibitem{GANetal}
I.~Gandarilla, V.~Santib{\'a}{\~n}ez, J.~Sandoval, and R.~Campa, ``Joint position regulation of a class of underactuated mechanical systems affected by lugre dynamic friction via the {IDA-PBC} method,'' \emph{International Journal of Control}, vol.~95, no.~6, pp. 1419--1431, 2022.

\bibitem{apriltag}
\BIBentryALTinterwordspacing
{University of Michigan EECS}, ``April{T}ag: A visual fiducial system,'' 2024, accessed: 2024-11-21. [Online]. Available: \url{https://april.eecs.umich.edu/software/apriltag}
\BIBentrySTDinterwordspacing

\bibitem{DEL2025}
C.~Della~Santina, ``Pushing the boundaries of actuators-to-task kinematic inversion: From fully actuated to underactuated (soft) robots,'' \emph{TechRxiv}, 2025.

\end{thebibliography}
\bibliographystyle{IEEEtran}

\end{document}